\documentclass{article}
\usepackage[left=0.8in, top=1in]{geometry}
\usepackage{amsmath,amsfonts}
\usepackage{algorithm,algpseudocode}
\usepackage{array}
\usepackage[caption=false,font=normalsize,labelfont=sf,textfont=sf]{subfig}
\usepackage{textcomp}
\usepackage{stfloats}
\usepackage{url}
\usepackage{verbatim}
\usepackage{graphicx}
\usepackage{cite}

\usepackage{multirow, multicol}
\usepackage{xcolor}
\usepackage{makecell}
\usepackage{booktabs}
\usepackage{xr}
\usepackage{xr-hyper}
\usepackage{hyperref}

\usepackage{amsmath,amsfonts,bm}
\usepackage{amsthm,amssymb}

\newcommand{\df}[1]{\mathrm{d}#1}
\newcommand{\e}{\mathrm{e}}

\def\rvc{{\mathbf{c}}}
\def\rvd{{\mathbf{d}}}
\def\rve{{\mathbf{e}}}
\def\rvf{{\mathbf{f}}}

\def\rvh{{\mathbf{h}}}
\def\rvu{{\mathbf{i}}}

\def\rvk{{\mathbf{k}}}

\def\rvn{{\mathbf{n}}}

\def\rvr{{\mathbf{r}}}
\def\rvs{{\mathbf{s}}}
\def\rvt{{\mathbf{t}}}
\def\rvu{{\mathbf{u}}}

\def\rvx{{\mathbf{x}}}
\def\rvy{{\mathbf{y}}}
\def\rvz{{\mathbf{z}}}

\def\mSigma{{\bm{\Sigma}}}

\def\mA{{\bm{A}}}
\def\mB{{\bm{B}}}

\def\mH{{\bm{H}}}
\def\mI{{\bm{I}}}

\def\mV{{\bm{V}}}
\def\mW{{\bm{W}}}

\def\mSigma{{\bm{\Sigma}}}

\DeclareMathAlphabet{\mathsfit}{\encodingdefault}{\sfdefault}{m}{sl}
\SetMathAlphabet{\mathsfit}{bold}{\encodingdefault}{\sfdefault}{bx}{n}

\def\gD{{\mathcal{D}}}
\def\gE{{\mathcal{E}}}

\def\gH{{\mathcal{H}}}
\def\gI{{\mathcal{I}}}

\def\gN{{\mathcal{N}}}

\def\gS{{\mathcal{S}}}

\def\sI{{\mathbb{I}}}

\def\sN{{\mathbb{N}}}

\def\sR{{\mathbb{R}}}

\newcommand{\E}{\mathbb{E}}

\newcommand{\KL}{D_{\mathrm{KL}}}

\newcommand{\Tr}{\mathrm{Tr}}

\usepackage[autostyle=true]{csquotes}

\theoremstyle{plain}
\newtheorem{theorem}{Theorem}
\newtheorem*{theorem*}{Theorem}  

\newtheorem{proposition}{Proposition}
\newtheorem{assumption}{Assumption}

\theoremstyle{definition}

\theoremstyle{remark}
\newtheorem*{remark}{Remark}

\begin{document}

\title{Unpaired Joint Distribution Modeling via Multi-Scale Image Representations}

\author{
Yihang Zou\thanks{Yau Mathematical Sciences Center and Department of Mathematical Sciences, Tsinghua University, Beijing, China {({zou-yh21@mails.tsinghua.edu.cn}).}},
Hui Zhang\thanks{Qiuzhen College, Tsinghua University, Beijing, China ({zhanghui23@mails.tsinghua.edu.cn}).}, 
Zuowei Shen\thanks{Department of Mathematics, National University of Singapore, Singapore ({matzuows@nus.edu.sg}).},
Chenglong Bao\thanks{Yau Mathematical Sciences Center, Tsinghua University, Beijing, China, Beijing Institute of Mathematical Sciences and Applications, Beijing, China, and State Key Laboratory of Membrane Sciences, Tsinghua University, Beijing, China ({clbao@tsinghua.edu.cn}).}
}
\date{}



\maketitle

\begin{abstract}
\textbf{This paper studies the problem of learning a joint distribution from marginal observations, which is inherently ill-posed due to the ambiguity of feasible couplings. We propose LUD-MSR, a latent-variable probabilistic framework that models the joint distribution via auxiliary representations and optimizes evidence lower bounds using only marginal data. Under mild assumptions, we establish an upper bound on the distribution approximation error. This analysis reveals a trade-off in representation learning between domain consistency and information preservation.
To address this trade-off, we introduce a Multi-Scale image Representation (MSR) mapping that exploits structural similarity at coarse scales while suppressing domain-specific variations. We show that MSR achieves a more favorable balance of this trade-off compared to existing approaches. Experiments on real-world denoising benchmarks, including cryo-electron microscopy (cryo-EM), demonstrate the effectiveness of the proposed framework.}
\end{abstract}


\section{Introduction}
{T}{he} success of deep learning critically depends on the availability of high-quality paired training data~\cite{lecun2015deep,krizhevsky2012imagenet}. However, in many scientific and engineering applications, collecting aligned samples $(\rvx, \rvy) \sim q_{X,Y}$ is expensive or infeasible~\cite{stuart2019comprehensive,hoffman2018cycada}. This challenge has motivated approaches that relax the requirement for paired supervision, including self-supervised~\cite{huang2021neighbor2neighbor,pang2021recorrupted,chen2024exploring,lee2022apbsn,he2022masked,oquab2023dinov2} and unpaired~\cite{zhu2017unpaired,radford2015unsupervised,jiang2021enlightengan,ulyanov2018deep,doersch2015unsupervised} learning methods. In the latter scenario, where only the marginal distributions $\rvx \sim q_{X}$ and $\rvy \sim q_{Y}$ are available, the objective is to infer a joint distribution $p(\rvx,\rvy)$ capable of generating pseudo-paired samples for downstream network training.

Mathematically, unpaired learning is formulated as seeking a joint distribution $p(\rvx, \rvy)$ that approximates the true distribution $q_{X,Y}$ while satisfying the marginal constraints
\begin{equation}
\label{eq:marginal_constraint}
\int p(\rvx, \rvy) \,\mathrm{d}\rvy = q_{X}(\rvx), \quad
\int p(\rvx, \rvy) \,\mathrm{d}\rvx = q_{Y}(\rvy).
\end{equation}
This problem is ill-posed because the set of joint distributions satisfying \eqref{eq:marginal_constraint} is generally not unique. Common strategies to alleviate this issue include regularizing $\rvx$ and $\rvy$ via discriminative terms~\cite{jang2021c2n,chen2022unpaired}, cycle-consistency losses~\cite{parmar2024one,meng2023cyclenet,zhu2017unpaired}, contrastive objectives~\cite{park2020contrastive,liu2021dual,hu2022qs}, or structural assumptions on latent variables (e.g., independence between domain-shared and domain-specific factors~\cite{wolf2021deflow,zheng2022learn,huang2018multimodal,lee2020drit}). While empirically successful across various tasks, these methods often rely on heuristic designs and lack theoretical justification, potentially leading to unstable or inconsistent pseudo-paired sample generation.
To address these challenges, we propose LUD-MSR, an interpretable framework for modeling joint distributions from unpaired data. We introduce a probabilistic graphical model to characterize the joint distribution, utilizing two auxiliary variables, $\rvh_\rvx$ and $\rvh_\rvy$, to represent $\rvx$ and $\rvy$, respectively. Within this graph, we derive training objectives using the evidence lower bounds (ELBOs) for the likelihoods $\log p(\rvx|\rvh_\rvx)$ and $\log p(\rvy|\rvh_\rvy)$, requiring only samples from $q_{X}$ and $q_{Y}$. Under specific assumptions, we characterize the error between the true joint distribution and our derived model. This analysis reveals an inherent trade-off between \emph{domain consistency} and \emph{information preservation},
directly motivating our structural design principles for $\rvh_{\rvx}$ and $\rvh_{\rvy}$. For real-world image noise modeling, we explicitly construct $\rvh_{\rvx}$ and $\rvh_{\rvy}$ via a Multi-Scale image Representation (MSR) mapping and theoretically demonstrate that MSR achieves a superior balance of this trade-off compared to existing approaches. Furthermore, our framework can be naturally extended for the semi-supervised case. Through a clean-to-noisy generation pipeline, LUD-MSR synthesizes realistic paired samples for downstream denoising tasks. Our main contributions are summarized as

\noindent \textbf{(i)} We propose LUD-MSR, a novel unpaired joint distribution modeling framework. Theoretical analysis demonstrates that LUD-MSR achieves a superior trade-off between domain consistency and information preservation,
thereby significantly reducing the distribution approximation error compared to existing methods.

\noindent \textbf{(ii)} LUD-MSR synthesizes realistic pseudo-training pairs for real-world denoising tasks. It surpasses existing noise modeling approaches on standard denoising benchmarks and demonstrates the advantages for the problem of cryo-EM image denoising.

\section{Related Works}
\noindent\textbf{Unpaired joint distribution modeling.}
Given the scarcity of paired data, unpaired joint distribution modeling seeks to learn the joint distribution $q_{X,Y}$ solely from the marginals $\rvx \sim q_{X}$ and $\rvy \sim q_{Y}$. To alleviate this inherent ill-posedness, early approaches~\cite{zhu2017unpaired, hoffman2018cycada} employed cycle-consistent adversarial networks. However, strict pixel-level cycle constraints often restrict generation diversity and induce geometric distortions. 
To address these limitations, subsequent works have maximized mutual information via contrastive learning, utilizing patch-wise frameworks~\cite{park2020contrastive}, dual objectives~\cite{liu2021dual}, or dynamic anchor sampling~\cite{hu2022qs}. 
Beyond contrastive methods, disentangled representations~\cite{huang2018multimodal, lee2020drit} explicitly promote diversity by decomposing the latent space into domain-shared content and domain-specific styles. This paradigm has been further augmented with physical priors, such as density and depth decomposition~\cite{yang2022self}, to handle complex real-world degradations~\cite{chen2022unpaired}. 
Recently, unpaired learning has shifted toward mathematically rigorous optimal transport (OT) and diffusion models. 
Representative methods include approximating the OT plan via shortest path regularization~\cite{xie2023unpaired} and formulating high-fidelity translation with Schr\"odinger Bridges~\cite{su2022dual}. Modern architectures also revisit classic regularization strategies, such as integrating cycle consistency into text-guided diffusion~\cite{meng2023cyclenet} or incorporating adversarial training for single-step generation~\cite{parmar2024one}.
Nevertheless, developing efficient joint distribution models that incorporate both interpretability and task-specific inductive biases remains a critical open challenge.

\noindent\textbf{Noise modeling.}
Real-world noise exhibits spatial correlations and complex patterns arising from the intractable image signal processing (ISP) pipeline. Although physics-informed approaches simulate real-world noise~\cite{guo2019toward,brooks2019unprocessing}, significant domain shifts persist due to hand-crafted model limitations. Consequently, research shifted toward learning-based noise modeling. DANet~\cite{yue2020dual} jointly optimizes a denoiser and a noise generator via adversarial training. Flow-sRGB~\cite{kousha2022modeling} employs normalizing flows (NFs) to capture noise distributions using camera metadata. NeCA~\cite{fu2023srgb} introduces a neighborhood-correlation network to model spatial dependencies, while NAFlow~\cite{kim2024srgb} utilizes NFs to represent camera-specific noise as Gaussian mixtures. However, these supervised methods remain constrained by paired data scarcity. Unpaired noise modeling mitigates this reliance, operating as a special case of unpaired joint distribution modeling where $q_{X}$ and $q_{Y}$ denote clean and noisy image distributions, respectively. C2N~\cite{jang2021c2n} proposes an adversarial clean-to-noisy generation framework. DeFlow~\cite{wolf2021deflow} uses NFs to decouple noise and signals in latent space, whereas LUD-VAE~\cite{zheng2022learn} employs a probabilistic graphical model with inference invariance enforced via linear diffusion. Nevertheless, designing effective latent representations and model priors to fully exploit inherent joint distribution correlations remains an open challenge.

\noindent\textbf{Multi-scale image representation.}
Because \emph{paired clean and noisy images share content features at coarser scales while exhibiting distinct degradation statistics at finer scales}~\cite{liu2018multi,ruderman1993statistics}, employing multi-scale image representations (e.g., wavelet decomposition~\cite{mallat1989theory}) effectively disentangles signal structures from noise across frequency bands. However, this approach presents an inherent trade-off: operating at coarser scales increases domain consistency but discards fine structural details, diminishing the original signal representation. Consequently, optimally balancing these competing objectives remains a pivotal challenge in unpaired noise modeling.

\section{Methodology}
This section is organized as follows. We first introduce LUD-MSR, an unpaired joint distribution modeling approach based on a novel probabilistic graphical model that utilizes auxiliary variables $\rvh_\rvx$ and $\rvh_\rvy$ to represent $\rvx$ and $\rvy$, respectively. For the specific application of image noise modeling, we detail the explicit construction of $\rvh_\rvx$ and $\rvh_\rvy$ via a Multi-Scale image Representation (MSR) mapping. Finally, we extend LUD-MSR to accommodate semi-supervised scenarios.

Throughout this paper, we assume $\rvx, \rvy \in \sR^{K}$. For paired samples $(\rvx, \rvy) \sim q_{X,Y}$, we define the residual as $\rvn = \rvy - \rvx$, which follows the marginal distribution $q_{\rvn}$. The covariance matrices $\mSigma_{X} = \E_{q_{X}(\rvx)}[\rvx\rvx^{\top}]$ and $\mSigma_{\rvn} = \E_{q_{\rvn}(\rvn)}[\rvn\rvn^{\top}]$ are assumed to be well-defined. The notation $f(x) \lesssim g(x)$ (or $f(x) \gtrsim g(x)$) indicates that $f(x) \leq C \, g(x)$ (or $f(x) \geq C \, g(x)$) for some constant $C > 0$. 

\begin{figure*}
\centering
\includegraphics[width=\linewidth]{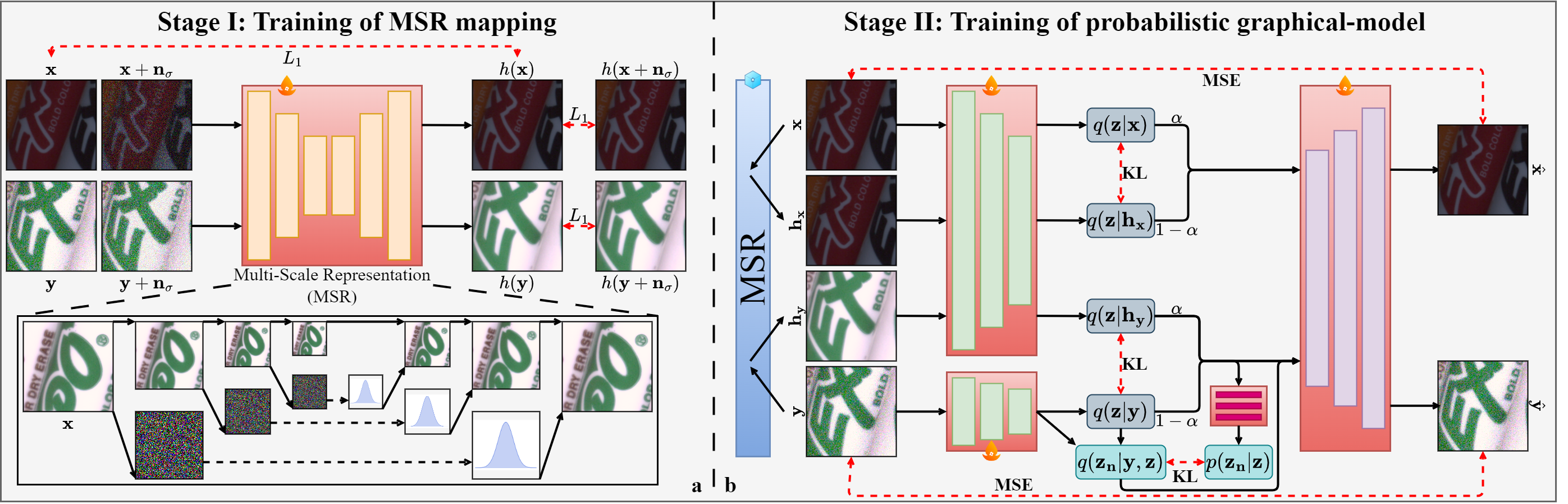}
\caption{Overview of the LUD-MSR training pipeline. \textbf{(a)} Stage I: Training the Multi-Scale image Representation (MSR) mapping. \textbf{(b)} Stage II: Training the probabilistic graphical model. The parameters of the pre-trained MSR mapping remain frozen during this stage.}
\label{fig:lr2flow_arch}
\end{figure*}

\subsection{The Design of LUD-MSR}
\begin{figure}
\centering
\includegraphics[width=0.6\linewidth]{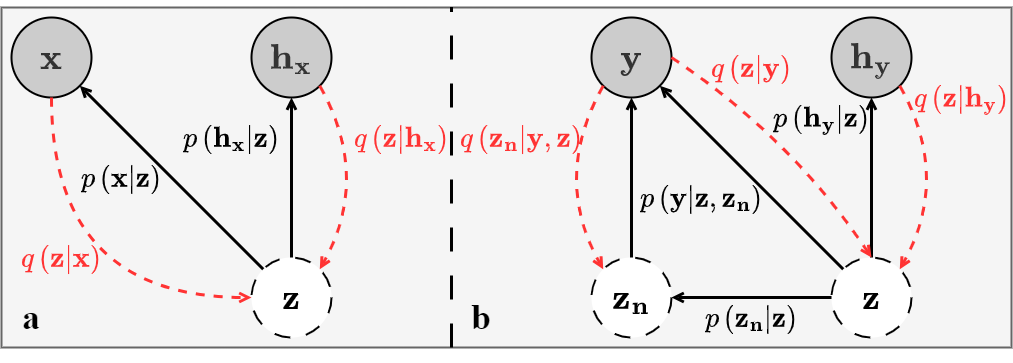}
\caption{Probabilistic graphical model of LUD-MSR. The generative processes for the $\rvx\sim q_{X}$ and $\rvy\sim q_{Y}$ are illustrated in \textbf{(a)} and \textbf{(b)}, respectively.}
\label{fig:SIR-NM_unpaired}
\end{figure}

Given the marginal distributions $q_X$ and $q_Y$, LUD-MSR aims to recover the true joint distribution $q_{X,Y}$. We first propose a latent-variable probabilistic graphical model to characterize the generative processes of $\rvx \sim q_{X}$ and $\rvy \sim q_{Y}$, as illustrated in Figure~\ref{fig:SIR-NM_unpaired}. Here, $\rvh_{\rvx}$ and $\rvh_{\rvy}$ denote \emph{frozen} auxiliary representations associated with $\rvx$ and $\rvy$, respectively. These variables are specifically designed to bridge the implicit coupling between the two domains, enabling joint distribution modeling. Their explicit construction is detailed in Section~\ref{sec:msr}. 
We introduce two latent variables, $\rvz$ and $\rvz_{\rvn}$, representing the domain-shared content and the $\rvy$-specific component, respectively. We assume that $\rvx$, $\rvh_{\rvx}$, and $\rvh_{\rvy}$ are generated solely by $\rvz$, whereas $\rvy$ is generated jointly by $\rvz$ and $\rvz_{\rvn}$.To approximate the intractable posterior distributions, we design the following inference models
\begin{equation}
\label{eq:inference_models_unpaired}
\begin{aligned}
& q(\rvz | \rvx, \rvh_{\rvx}) = \alpha q(\rvz | \rvx) + (1-\alpha) q(\rvz | \rvh_{\rvx}), \\
& q(\rvz | \rvy, \rvh_{\rvy}) = \alpha q(\rvz | \rvy) + (1-\alpha) q(\rvz | \rvh_{\rvy}), \\
& q(\rvz, \rvz_{\rvn} | \rvy, \rvh_{\rvy}) = q(\rvz | \rvy, \rvh_{\rvy}) q(\rvz_{\rvn} | \rvy, \rvz),
\end{aligned}
\end{equation}
where $\alpha \in [0,1]$. 
Furthermore, we derive the $\rvx$-to-$\rvy$ and $\rvy$-to-$\rvx$ conditional generation pipelines from our probabilistic graphical model as
\begin{equation}
\label{eq:conditional_generation}
\begin{aligned}
p(\rvy | \rvx) &= \E_{q(\rvz|\rvx) p(\rvz_{\rvn} | \rvz)}p(\rvy | \rvz, \rvz_{\rvn}), \\
p(\rvx | \rvy) &= \E_{q(\rvz | \rvy)}p(\rvx | \rvz), 
\end{aligned}
\end{equation}
which in turn induce the joint generative models
\begin{equation}
\label{eq:lud-msr_joint}
\begin{aligned}
p_{\text{X$\rightarrow$Y}}(\rvx,\rvy) & = q_{X}(\rvx)p(\rvy | \rvx), \\
p_{\text{Y$\rightarrow$X}}(\rvx,\rvy) & = q_{Y}(\rvy)p(\rvx | \rvy).
\end{aligned}
\end{equation}
Within our probabilistic graphical model, the inference and generative models are parameterized as Gaussian networks. Detailed architectural specifications are provided in Appendix~\ref{appendix:architecture}.
In the following, we quantify the discrepancy between our generative framework and the true joint distribution under mild assumptions.

\begin{assumption}
\label{assumption:lud-msr}
\noindent (i) There exist constants $\varepsilon_0>0$ and $\eta>0$ such that for $(\rvx, \rvy)\sim q_{X,Y}$, the following inequalities hold
\begin{equation}
\label{eq:p_regularity1}
\e^{-\eta}\leq\frac{p(\rvy|\rvx)}{p(\rvy|\rvh_{\rvy})}\leq \e^{\eta} \text{ if } \min\{p(\rvy|\rvx),p(\rvy|\rvh_{\rvy})\} \leq \varepsilon_0,
\end{equation}
\begin{equation}
\label{eq:p_regularity2}
\e^{-\eta}\leq\frac{p(\rvx | \rvy)}{p(\rvx | \rvh_{\rvx})}\leq \e^{\eta} \text{ if } \min \{p(\rvx | \rvy), p(\rvx | \rvh_{\rvx})\} \leq \varepsilon_{0}.
\end{equation}
\noindent (ii) Let $\varphi(\rvz; \rvy) = \E_{p(\rvz_{\rvn} | \rvz)}p(\rvy | \rvz,\rvz_{\rvn})$, assume that $\|\varphi(\rvz; \rvy)\|_{\infty}$ is essentially bounded with respect to $\rvy \sim q_{Y}$, and that $\|p(\rvx | \rvz)\|_{\infty}$ is essentially bounded with respect to $\rvx\sim q_{X}$.
\end{assumption}
\begin{remark}
Assumption~\ref{assumption:lud-msr}(i) imposes a tail-equivalence condition via \eqref{eq:p_regularity1} and \eqref{eq:p_regularity2}, bounding the low-probability discrepancies between $p(\rvy | \rvx)$ and $p(\rvy | \rvh_{\rvy})$, as well as between $p(\rvx | \rvy)$ and $p(\rvx | \rvh_{\rvx})$.
This precludes pathological tail behaviors and regularizes the generative mappings. Additionally, Assumption~\ref{assumption:lud-msr}(ii) posits essentially bounded conditional densities for $\rvx$ and $\rvy$ given $\rvz$ to guarantee model stability.
\end{remark}
Under this assumption, we establish a bound on the discrepancy between the generated and true joint distributions, as formalized in the following theorem. 

\begin{theorem}
\label{thm:dist_error}
Let $p_{\text{X$\to$Y}}$ and $p_{\text{Y$\to$X}}$ be defined as in~\eqref{eq:lud-msr_joint}. Suppose the Assumption~\ref{assumption:lud-msr} holds, we have
\begin{equation}
\label{eq:kl_upper_bound}
\begin{aligned}
& \KL(q_{X,Y} || p_{\text{X$\to$Y}}) + \KL(q_{X,Y} || p_{\text{Y$\to$X}}) \\
\lesssim & -\E_{q_{X}(\rvx)}\log p(\rvx | \rvh_{\rvx}) - \E_{q_{Y}(\rvy)}\log p(\rvy | \rvh_{\rvy}) \\
& + \sqrt{\E_{q_{X}(\rvx)}\KL(q(\rvz | \rvx) || q(\rvz | \rvh_{\rvx}))} 
+ \sqrt{\E_{q_{Y}(\rvy)}\KL(q(\rvz | \rvy) || q(\rvz | \rvh_{\rvy}))} \\
& + \E_{q_{X,Y}(\rvx,\rvy)}\|q(\rvz | \rvh_{\rvx}) - q(\rvz | \rvh_{\rvy})\|_1.
\end{aligned}
\end{equation}
\end{theorem}

\begin{proof}[Proof of Theorem~\ref{thm:dist_error}]
Our proof strategy consists of the following steps: (i) bound the log-likelihood discrepancy between $\log p(\rvy | \rvx)$ and $\log p(\rvy | \rvh_{\rvy})$ via $|p(\rvy | \rvx) - p(\rvy | \rvh_{\rvy})|$, and bound the discrepancy between $\log p(\rvx | \rvy)$ and $\log p(\rvx | \rvh_{\rvx})$ via $|p(\rvx | \rvy) - p(\rvx | \rvh_{\rvx})|$; (ii) estimate the upper bounds of $\E_{q_{X,Y}(\rvx, \rvy)}|p(\rvy | \rvx) - p(\rvy | \rvh_{\rvy})|$ and $\E_{q_{X,Y}(\rvx, \rvy)}|p(\rvx | \rvy) - p(\rvx | \rvh_{\rvx})|$; and (iii) control the distribution discrepancies $\KL(q_{X,Y} || p_{\text{X$\to$Y}})$ and $\KL(q_{X,Y} || p_{\text{Y$\to$X}})$ using the above results. \\
\textbf{(i)} Invoking the regularity condition in~\eqref{eq:p_regularity1}, we bound the log-likelihood discrepancy between $\log p(\rvy | \rvx)$ and $\log p(\rvy | \rvh_{\rvy})$ as
\begin{equation}
\label{eq:log_p_diff1}
\begin{aligned}
\left|\log p(\rvy | \rvx) - \log p(\rvy | \rvh_{\rvy})\right| 
\leq & \left|\log p(\rvy | \rvx) - \log p(\rvy | \rvh_{\rvy})\right|
\cdot \sI(\min\{p(\rvy | \rvx), p(\rvy | \rvh_{\rvy})\} \leq \varepsilon_{0}) \\
& + \left|\log p(\rvy | \rvx) - \log p(\rvy | \rvh_{\rvy})\right|
\cdot \sI(\min\{p(\rvy | \rvx), p(\rvy | \rvh_{\rvy})\} > \varepsilon_{0}) \\
\leq & \eta + \varepsilon_{0}^{-1} \left|p(\rvy | \rvx) - p(\rvy | \rvh_{\rvy})\right|,
\end{aligned}
\end{equation}
where $\sI(\cdot)$ denotes the indicator function. Similarly, by~\eqref{eq:p_regularity2}, we obtain
\begin{equation}
\label{eq:log_p_diff2}
\begin{aligned}
\left|\log p(\rvx | \rvy) - \log p_{}(\rvx | \rvh_{\rvx})\right|
\leq & \eta + \varepsilon_{0}^{-1} \left|p(\rvx | \rvy) - p(\rvx | \rvh_{\rvx})\right|.
\end{aligned}
\end{equation}
\textbf{(ii)} For any $(\rvx, \rvy) \sim q_{X,Y}$, we quantify the discrepancies between $p(\rvy | \rvx)$ and $p(\rvy | \rvh_{\rvy})$, and symmetrically between $p(\rvx | \rvy)$ and $p(\rvx | \rvh_{\rvx})$. Given the generative process of our proposed graphical model, we obtain
\begin{equation*}
\begin{aligned}
\left|p(\rvy | \rvh_{\rvy}) - p(\rvy | \rvx)\right| 
= & {\textstyle \left|\int \left(q(\rvz | \rvh_{\rvy}) - q(\rvz | \rvx)\right)p(\rvz_{\rvn} | \rvz)p(\rvy | \rvz, \rvz_{\rvn})\df \rvz \df \rvz_{\rvn}\right|} \\
= & {\textstyle \left|\int \varphi(\rvz; \rvy) \left(q(\rvz | \rvx) - q(\rvz | \rvh_{\rvy})\right)\df \rvz\right|} \\
\leq & \|\varphi(\rvz; \rvy)\|_{\infty}\cdot\left\|q(\rvz | \rvx)-q(\rvz | \rvh_{\rvy})\right\|_{1}.
\end{aligned}
\end{equation*}
Symmetrically, we have
\begin{equation*}
\begin{aligned}
\left|p(\rvx | \rvh_{\rvx}) - p(\rvx | \rvy)\right| 
= & {\textstyle \left|\int (q(\rvz | \rvh_{\rvx}) - q(\rvz | \rvy))p(\rvx | \rvz)\df\rvz\right|} \\
\leq & \|p(\rvx | \rvz)\|_{\infty} \cdot \|q(\rvz | \rvh_{\rvx}) - q(\rvz | \rvy)\|_{1}.
\end{aligned}
\end{equation*}
Applying the triangle inequality to the L1 norm, we obtain
\begin{equation}
\label{eq:p_diff}
\begin{aligned}
\left|p(\rvy | \rvh_{\rvy}) - p(\rvy | \rvx)\right| 
\leq & \|\varphi(\rvz; \rvy)\|_{\infty} \, (
\|q(\rvz | \rvx) - q(\rvz|\rvh_{\rvx})\|_1 + \|q(\rvz|\rvh_{\rvx}) - q(\rvz|\rvh_{\rvy})\|_1) \\
\left|p(\rvx | \rvh_{\rvx}) - p(\rvx | \rvy)\right| 
\leq & \|p(\rvx |\rvz)\|_{\infty} \, (
\|q(\rvz | \rvy) - q(\rvz|\rvh_{\rvy})\|_1 + \|q(\rvz|\rvh_{\rvx}) - q(\rvz|\rvh_{\rvy})\|_1).
\end{aligned}
\end{equation}
\textbf{(iii)} According to~\eqref{eq:log_p_diff1} and~\eqref{eq:log_p_diff2}, we obtain the  conditional KL divergence bounds
\begin{equation}
\label{eq:kl_conditioal_dist_x}
\begin{aligned}
& \KL(q_{Y|X}(\rvy | \rvx) || p(\rvy | \rvx)) \\
= & \KL(q_{Y|X}(\rvy | \rvx) || p(\rvy | \rvh_{\rvy})) + \E_{q_{Y|X}(\rvy | \rvx)}\log \tfrac{p(\rvy | \rvh_{\rvy})}{p(\rvy | \rvx)} \\
\leq & -H(q_{Y|X}) - \E_{q_{Y|X}(\rvy | \rvx)}\log p(\rvy | \rvh_{\rvy}) + \eta 
+ \varepsilon_{0}^{-1} \E_{q_{Y|X}(\rvy | \rvx)}\left|p(\rvy | \rvx) - p(\rvy | \rvh_{\rvy})\right|
\end{aligned}
\end{equation}
for $\rvx\sim q_{X}$, and
\begin{equation*}
\begin{aligned}
& \KL(q_{X|Y}(\rvx | \rvy) || p(\rvx | \rvy)) \\
= & \KL(q_{X|Y}(\rvx | \rvy) || p(\rvx | \rvh_{\rvx})) + \E_{q_{X|Y}(\rvx | \rvy)}\log\tfrac{p(\rvx | \rvh_{\rvx})}{p(\rvx | \rvy)} \\
\leq & -H(q_{X|Y}) - \E_{q_{X|Y}(\rvx | \rvy)}\log p(\rvx | \rvh_{\rvx}) + \eta 
+ \varepsilon_{0}^{-1}\E_{q_{X|Y}(\rvx | \rvy)}\left|p(\rvx | \rvy) - p(\rvx | \rvh_{\rvx})\right|
\end{aligned}
\end{equation*}
for $\rvy\sim q_{Y}$, where $H(p) = -\E_{p(\rvx)}\log p(\rvx)$ denotes the Shannon entropy of the distribution $p$.
Recalling the definitions of $p_{\text{X$\to$Y}}$ and $p_{\text{Y$\to$X}}$ in~\eqref{eq:lud-msr_joint}, the joint KL divergences in~\eqref{eq:kl_upper_bound} are expressed as
\begin{equation*}
\begin{aligned}
& \KL(q_{X,Y} || p_{\text{X$\to$Y}}) = \E_{q_{X}(\rvx)}\KL(q_{Y|X}(\rvy | \rvx) || p(\rvy | \rvx)), \\
& \KL(q_{X,Y} || p_{\text{Y$\to$X}}) = \E_{q_{Y}(\rvy)}\KL(q_{X|Y}(\rvx | \rvy) || p(\rvx | \rvy)).
\end{aligned}
\end{equation*}
Combining~\eqref{eq:p_diff} with Pinsker's inequality yields the following upper bounds
\begin{equation*}
\begin{aligned}
\KL(q_{X,Y} || p_{\text{X$\to$Y}}) 
\leq & -H(q_{Y|X}) + \eta - \E_{q_{Y}(\rvy)}\log p(\rvy | \rvh_{\rvy}) \\
& + \sqrt{2}\varepsilon_{0}^{-1} \E_{q_{Y}(\rvy)}\|\varphi(\rvz; \rvy)\|_{\infty}\sqrt{\E_{q_{X}(\rvx)}\KL(q(\rvz | \rvx) || q(\rvz | \rvh_{\rvx}))} \\
& + \varepsilon_{0}^{-1} \E_{q_{X,Y}(\rvx,\rvy)}\|\varphi(\rvz; \rvy)\|_{\infty}\|q(\rvz | \rvh_{\rvx}) - q(\rvz | \rvh_{\rvy})\|_{1}, 
\end{aligned}
\end{equation*}
\begin{equation*}
\begin{aligned}
\KL(q_{X,Y} || p_{\text{Y$\to$X}}) 
\leq & -H(q_{X|Y}) + \eta - \E_{q_{X}(\rvx)}\log p(\rvx | \rvh_{\rvx}) \\
& + \sqrt{2}\varepsilon_{0}^{-1}
\E_{q_{X}(\rvx)}\|p(\rvx | \rvz)\|_{\infty}
\sqrt{\E_{q_{Y}(\rvy)}\KL(q(\rvz | \rvy) || q(\rvz | \rvh_{\rvy}))} \\
& + \varepsilon_{0}^{-1}
\E_{q_{X,Y}(\rvx,\rvy)}\|p(\rvx | \rvz)\|_{\infty}
\|q(\rvz | \rvh_{\rvx}) - q(\rvz | \rvh_{\rvy})\|_{1},
\end{aligned}
\end{equation*}
Note that the terms $-H(q_{Y|X})$, $-H(q_{X|Y})$, and $\eta$
are constant with respect to the optimization of the generative model.
\end{proof}

\noindent\textbf{The loss function.} Theorem~\ref{thm:dist_error} motivates us to train the proposed probabilistic graphical model by minimizing the first four terms in~\eqref{eq:kl_upper_bound}. Using the variational inference framework, we derive the ELBOs for $\log p(\rvx | \rvh_{\rvx})$ and $\log p(\rvy | \rvh_{\rvy})$ as
\begin{equation}
\label{eq:elbo_x}
\begin{aligned}
\log p(\rvx | \rvh_{\rvx}) \geq & \,\E_{q(\rvz | \rvx, \rvh_{\rvx})}\log p(\rvx | \rvz) 
- \alpha \, \KL(q(\rvz | \rvx) || q(\rvz | \rvh_{\rvx})) = \, \text{ELBO}_{\rvx},
\end{aligned}
\end{equation}
\begin{equation}
\label{eq:elbo_y}
\begin{aligned}
\log p(\rvy | \rvh_{\rvy}) \geq & \,\E_{q(\rvz,\rvz_{\rvn} | \rvy, \rvh_{\rvy})} \log p(\rvy | \rvz, \rvz_{\rvn}) 
- \alpha \, \KL(q(\rvz | \rvy) || q(\rvz | \rvh_{\rvy})) \\
& - \E_{q(\rvz | \rvy, \rvh_{\rvy})}\KL(q(\rvz_{\rvn} | \rvy, \rvz) || p(\rvz_{\rvn} | \rvz)) 
= \, \text{ELBO}_{\rvy}.
\end{aligned}
\end{equation}
A detailed derivation of~\eqref{eq:elbo_x} and~\eqref{eq:elbo_y} is provided in Appendix~\ref{appendix:elbo}. Since~\eqref{eq:elbo_x} and~\eqref{eq:elbo_y} contain the third and fourth term in \eqref{eq:kl_upper_bound}, the loss function is defined as
\begin{equation}
\label{eq:loss_graph_unpaired}
L_{\text{g}}^{\text{(up)}} = - \left( \E_{q_{X}(\rvx)} \text{ELBO}_{\rvx} + \E_{q_{Y}(\rvy)} \text{ELBO}_{\rvy} \right).
\end{equation}
Consequently, the evaluation of \eqref{eq:loss_graph_unpaired} only requires samples from marginal distributions $q_X$ and $q_Y$.

Notably, the final term in~\eqref{eq:kl_upper_bound} (namely, the \emph{inference invariance} term) depends on the auxiliary representations $\rvh_\rvx$ and $\rvh_\rvy$. Consequently, these representations must be carefully designed to mitigate this error prior to training the probabilistic graphical model. The design of these auxiliary representations is guided by the following two principles
\begin{align}
& \text{[\textbf{domain consistency}] $\rvh_{\rvx} = \rvh_{\rvy}$ for $(\rvx, \rvy) \sim q_{X,Y}$.} \label{eq:msr_domain_consistency} \\
& \text{[\textbf{information preservation}] $\rvh_{\rvx} = \rvx$ for $\rvx \sim q_{X}$.} \label{eq:msr_information_preservation}
\end{align}
The domain consistency condition is intended to effectively reduce the inference invariance error. However, enforcing this condition in isolation can lead to degenerate solutions (e.g., $\rvh_{\rvx} = \rvh_{\rvy} \equiv \mathbf{0}$ for all $\rvx \sim q_{X}$ and $\rvy \sim q_{Y}$). This observation naturally motivates the information preservation requirement. We detail the construction of $\rvh_\rvx$ and $\rvh_\rvy$ and analyze the trade-off between these two competing requirements in the subsequent subsection.
In summary, the LUD-MSD framework consists of two sequential stages: \\
{\bf Stage I:} Learn the auxiliary representations $\rvh_\rvx$ and $\rvh_\rvy$ using~\eqref{eq:lr2flow_objective_unpaired}. \\
{\bf Stage II:} With $\rvh_\rvx$ and $\rvh_\rvy$ fixed, train the probabilistic graphical model using~\eqref{eq:loss_graph_unpaired}.

The overall LUD-MSR training pipeline is illustrated in Figure~\ref{fig:lr2flow_arch} and detailed in Algorithm~\ref{alg:sir-nm_train}.

\begin{algorithm}
\caption{LUD-MSR training (unpaired)}
\label{alg:sir-nm_train}
\begin{algorithmic}[1]
\Require Marginal distributions $q_{X}$ and $q_{Y}$
\While{$h$ has not converged} \Comment{\textbf{Stage I}: Train auxiliary representations}
\State Sample $\rvx \sim q_{X}$ and $\rvy \sim q_{Y}$
\State Update $h$ by minimizing $L_{h}^{\text{(up)}}$ defined in~\eqref{eq:lr2flow_objective_unpaired}
\EndWhile
\State Freeze parameters of the mapping $h$
\While{$p_{\text{X$\to$Y}}$ has not converged} \Comment{\textbf{Stage II}: Train probabilistic graphical model}
\State Sample $\rvx \sim q_{X}$ and $\rvy \sim q_{Y}$
\State Compute $\rvh_{\rvx} = h(\rvx)$ and $\rvh_{\rvy} = h(\rvy)$
\State Update $p_{\text{X$\to$Y}}$ by minimizing $L_{\text{g}}^{\text{(up)}}$ defined in~\eqref{eq:loss_graph_unpaired}
\EndWhile
\State \Return Joint generative model $p_{\text{X$\to$Y}}(\rvx, \rvy)$
\end{algorithmic}
\end{algorithm}

\subsection{The Construction of Auxiliary Representations}
\label{sec:msr}

Given a hypothesis space $\gH$, we aim to find a mapping $h\in\gH$ such that $\rvh_{\rvx} = h(\rvx)$ and $\rvh_{\rvy} = h(\rvy)$. The domain consistency~\eqref{eq:msr_domain_consistency} and information preservation~\eqref{eq:msr_information_preservation} conditions naturally gives the loss function
\begin{equation*}
\min_{h\in\gH} \, \E_{q_{X}(\rvx)}\|\rvx - h(\rvx)\|_{1} + \E_{q_{X,Y}(\rvx,\rvy)}\|h(\rvx) - h(\rvy)\|_{1}.
\end{equation*}
Note that the second term is intractable due to the absence of paired samples. To approximate it, we use the result
\begin{equation}
\label{eq:h_inequality}
\begin{aligned}
\E_{q_{X,Y}(\rvx,\rvy)}\|h(\rvx) - h(\rvy)\|_{1} 
\leq & \E_{q_{X}(\rvx)p_{\sigma}(\rvn_{\rvx})}\|h(\rvx) - h(\rvx + \rvn_{\rvx})\|_{1} 
+ \E_{q_{Y}(\rvy)p_{\sigma}(\rvn_{\rvy})}\|h(\rvy) - h(\rvy + \rvn_{\rvy})\|_{1} \\
& + \E_{q_{X,Y}(\rvx,\rvy)p_{\sigma}(\rvn_{\rvx})p_{\sigma}(\rvn_{\rvy})}\|h(\rvx + \rvn_{\rvx}) - h(\rvy + \rvn_{\rvy})\|_{1}
\end{aligned}
\end{equation}
where $p_{\sigma} = \gN(\mathbf{0}, \sigma^{2}\mI)$.
While the first two terms in~\eqref{eq:h_inequality} can be explicitly minimized using unpaired data,
the third term becomes statistically indistinguishable as $\sigma \to \infty$.
\begin{proposition}
\label{prop:asympotic_consistency}
Assume that the mapping $h$ is bounded. Then
\begin{equation*}
\begin{aligned}
\lim_{\sigma\to\infty}\|\E_{q_{X,Y}(\rvx,\rvy) p_{\sigma}(\rvn_{\rvx})p_{\sigma}(\rvn_{\rvy})} [h(\rvx + \rvn_{\rvx}) - h(\rvy + \rvn_{\rvy})]\|_{1} 
=0.
\end{aligned}
\end{equation*}
\end{proposition}
The proof is given in Appendix~\ref{sec:app_proof_of_prop_1}.
Consequently, we formulate the surrogate loss as
\begin{equation}
\label{eq:lr2flow_objective_unpaired}
\begin{aligned}
L_{h}^{\text{(up)}} = & \E_{q_{X}(\rvx)}\|\rvx - h(\rvx)\|_{1} 
+ \E_{q_{X}(\rvx) p_{\sigma}(\rvn_{\rvx})}\|h(\rvx) - h(\rvx + \rvn_{\rvx})\|_{1} \\
& + \E_{q_{Y}(\rvy) p_{\sigma}(\rvn_{\rvy})}\|h(\rvy) - h(\rvy + \rvn_{\rvy})\|_{1}.
\end{aligned}
\end{equation}
While~\eqref{eq:lr2flow_objective_unpaired} provides a general objective for learning $h$ across arbitrary data modalities, we leverage task-specific structural properties to incorporate an explicit model prior through the design of $\gH$.

\noindent\textbf{Design of $\gH$ for image denoising.}
In the context of image denoising, we have $\rvy = \rvx + \rvn$ for $(\rvx, \rvy) \sim q_{X,Y}$, where $\rvx$, $\rvy$, and $\rvn$ denote the clean image, the noisy image, and the noise component, respectively. We now present the construction of the hypothesis space $\gH$ for this specific task. Given that the energy of natural images is predominantly concentrated in low frequencies~\cite{field1987relations} and that clean-noisy pairs exhibit significant structural similarity in low-frequency subbands~\cite{ruderman1993statistics}, a natural strategy is to design $h$ to extract compact, multi-scale low-frequency features. That is,
\begin{equation}
\label{eq:h_design}
\begin{aligned}
\gH = \{(\gD^{m})^{\top} \circ \gD^{m} \mid \gD:\rvx \mapsto \downarrow_{2}(\rvk_{\text{low}} \ast \rvx)\}
\end{aligned}
\end{equation}
where $\rvk_{\text{low}}$ represents a low-pass filter and $\downarrow_{2}$ denotes downsampling by a factor of 2. The following proposition quantitatively validates the trade-off between domain consistency~\eqref{eq:msr_domain_consistency} and information preservation~\eqref{eq:msr_information_preservation}.
\begin{proposition}
\label{prop:h_simple_design}
Let $\rvk_{\text{low}} = (\frac{1}{4}, \frac{1}{2}, \frac{1}{4})$ . If $\gH$ is defined as in~\eqref{eq:h_design} and $h\in\gH$, we obtain
\begin{equation*}
\begin{aligned}
& \E_{q_{X,Y}(\rvx, \rvy)}\|h(\rvx) - h(\rvy)\|_2^2 \lesssim K / 2^{3m},\\
& \E_{q_{X}(\rvx)}\|\rvx - h(\rvx)\|_{2}^2 \gtrsim K\, (1-2^{-3m}).
\end{aligned}
\end{equation*}
\end{proposition}

The proof is provided in Appendix~\ref{appendix:proof_h_simple_design}.
These results indicate that while increasing $m$ enhances the representation invariance between clean and noisy images, it simultaneously incurs information loss by attenuating high-frequency components. Consequently, the primary challenge is to maintain this invariance while preserving more high-frequency information within the mapping. Building upon our previous work~\cite{bao2026enhancing}, we define the hypothesis space $\gH$ as the Multi-Scale image Representation (MSR) mapping class
\begin{equation}
\label{eq:msr_class}
\gH_{\text{MSR}} = \{h: \rvx \mapsto f^{-1}([f(\rvx)]_{1:d}, \mathbf{0}) \mid d\in \sN, \, f\in \mathrm{Diff}(\sR^{K})\}
\end{equation}
where $\mathrm{Diff}(\sR^{K})$ denotes the space of diffeomorphisms on $\sR^{K}$, and the term $[f(\rvx)]_{1:d}$ represents the sub-vector corresponding to the first $d$ dimensions of $f(\rvx)$. Specifically, we implement the MSR mapping in practice as
\begin{equation}
\label{eq:msr_mapping}
h(\rvx) = f^{-1}([f(\rvx)]_{1:K/2^{T}}, \mathbf{0}),
\end{equation}
where the invertible mapping $f = f_{T} \circ \cdots \circ f_{1}$ is composed of $T$ scale levels, each level $f_{l} = g_{l} \circ \mW$ comprises a normalizing flow $g_{l}$ and a wavelet operator $\mW$. Compared to the general class defined in~\eqref{eq:msr_class}, we set $d=K/2^{T}$ to explicitly capture the low-frequency features at the $T$-th scale level. In our experiments, we set $T=3$ and employ the linear B-spline wavelet tight frame~\cite{daubechies2003framelets} as $\mW$. Detailed architectural specifications for the invertible mapping $f$ are provided in Appendix~\ref{sec:archi_msr}.

\begin{remark}
(i) The class $\gH_{\text{MSR}}$ defined in~\eqref{eq:msr_class} simplifies to the class of orthogonal projections
\begin{equation}
\label{eq:linear_lr2flow}
{\textstyle \gH_{\text{OP}} = \{h:\rvx \mapsto \mH\rvx \mid \mH \in \gS_{+}^{K} \text{, } \mH^{2} = \mH\}}.
\end{equation}
if we restrict $f\in O(K)$. This space serves as a linear surrogate for the MSR mappings. \\
(ii) In LUD-VAE~\cite{zheng2022learn}, the hypothesis space is constructed using noise injection mappings
\begin{equation}
\label{eq:lud-vae}
{\textstyle \gH_{\text{NI}}  = \{h: \rvx \mapsto \rvx + \rvn_{\sigma} \mid \rvn_{\sigma} \sim p_{\sigma} \text{, } \sigma>0\}.}
\end{equation}
Different hypothesis spaces induce different errors for the domain consistency~\eqref{eq:msr_domain_consistency} and information preservation~\eqref{eq:msr_information_preservation}. We will provide the theoretical analysis in Section~\ref{sec:balance_analysis}.
\end{remark}

\subsection{Extension to the Semi-Supervised Case}

LUD-MSR can be extended to the semi-supervised case, where the goal is to estimate the joint distribution $q_{X,Y}$ using unpaired data $\rvx \sim q_{X}$ and $\rvy \sim q_{Y}$, together with a limited set of paired samples ${(\rvx_{k}, \rvy_{k})} \sim q_{X,Y}$. In this setting, in addition to optimizing LUD-MSR with the unpaired objectives $L_{h}^{\text{(up)}}$ and $L_{\text{g}}^{\text{(up)}}$ defined in~\eqref{eq:lr2flow_objective_unpaired} and~\eqref{eq:loss_graph_unpaired}, we further incorporate the paired samples to enhance the overall training procedure. 

\noindent {\bf Stage I:} During the training of MSR mapping $h$, we directly use paired samples to minimize the domain consistency and information preservation errors:
\begin{equation}
\label{eq:lr2flow_objective_paired}
L_{h}^{\text{(p)}} = \E_{q_{X,Y}(\rvx, \rvy)} \left[\|h(\rvx) - h(\rvy)\|_{1} + \|\rvx - h(\rvx)\|_{1} \right].
\end{equation}
{\bf Stage II:} To train the probabilistic graphical model, we use paired samples to maximize the conditional log-likelihoods $\log p(\rvy | \rvx) + \log p(\rvx | \rvy)$. By setting 
\begin{equation}
\label{eq:inference_model_paired}
q(\rvz | \rvx, \rvy) = \alpha q(\rvz | \rvx) + (1-\alpha) q(\rvz | \rvy),
\end{equation}
we have
\begin{equation}
\label{eq:elbo_paired}
\begin{aligned}
\log p(\rvx | \rvy) \geq & \E_{q(\rvz | \rvx, \rvy)}\log p(\rvx | \rvz) - \alpha \KL(q(\rvz | \rvx) || q(\rvz | \rvy)), \\
\log p(\rvy | \rvx) \geq & \E_{q(\rvz | \rvx,\rvy)q(\rvz_{\rvn} | \rvy, \rvz)}\log p(\rvy | \rvz, \rvz_{\rvn}) 
- (1-\alpha)\KL(q(\rvz | \rvy) || q(\rvz | \rvx)) \\
& - \E_{q(\rvz | \rvx,\rvy)}\KL(q(\rvz_{\rvn} | \rvy, \rvz) || p(\rvz_{\rvn} | \rvz)).
\end{aligned}
\end{equation}
The derivation of~\eqref{eq:elbo_paired} is deferred to Appendix~\ref{appendix:elbo}. Therefore, the loss function for the paired samples is
\begin{equation}
\label{eq:sir-nm_paired}
\begin{aligned}
L^{\text{(p)}}_{\text{g}} = &
- \E_{q(\rvz | \rvx, \rvy)}\log p(\rvx | \rvz) 
- \E_{q(\rvz | \rvx,\rvy)q(\rvz_{\rvn} | \rvy, \rvz)}
\log p(\rvy | \rvz, \rvz_{\rvn}) \\
& + \alpha \KL(q(\rvz | \rvx) || q(\rvz | \rvy)) 
+ (1-\alpha)\KL(q(\rvz | \rvy) || q(\rvz | \rvx)) \\
& + \E_{q(\rvz | \rvx,\rvy)}\KL(q(\rvz_{\rvn} | \rvy, \rvz) || p(\rvz_{\rvn} | \rvz)).
\end{aligned}
\end{equation}
Please refer to Algorithm~\ref{alg:sir-nm_semi_train} for the training procedure of LUD-MSR in the semi-supervised setting.




\begin{remark}
In the semi-supervised setting, our training objective promotes accurate distribution modeling according to~\eqref{eq:kl_upper_bound} by bounding the inference invariance error
\begin{equation*}
\label{eq:infer_invariance_control}
\begin{aligned}
\|q(\rvz | \rvh_{\rvx}) - q(\rvz | \rvh_{\rvy})\|_{1} 
\lesssim & \sqrt{\KL(q(\rvz | \rvx) || q(\rvz | \rvy))}  + \sqrt{\KL(q(\rvz | \rvy) || q(\rvz | \rvx))} \\
& + \sqrt{\KL(q(\rvz | \rvx) || q(\rvz | \rvh_{\rvx}))}  + \sqrt{\KL(q(\rvz | \rvy) || q(\rvz | \rvh_{\rvy}))},
\end{aligned}
\end{equation*}
Specifically, the paired loss $L^{\text{(p)}}_{\text{(g)}}$ in~\eqref{eq:sir-nm_paired} minimizes the first two terms of this bound, while the unpaired loss $L^{\text{(up)}}_{\text{(g)}}$ in~\eqref{eq:loss_graph_unpaired} penalizes the latter two.
\end{remark}


\begin{algorithm}
\caption{LUD-MSR training (semi-supervised)}
\label{alg:sir-nm_semi_train}
\begin{algorithmic}[1]
\Require Marginal distributions $q_{X}$ and $q_{Y}$, and a limited set of paired samples $\{(\rvx_{k}, \rvy_{k})\} \sim q_{X,Y}$.
\While{$h$ has not converged} \Comment{\textbf{Stage I}: Train auxiliary representations}
\State Sample $\rvx \sim q_{X}$ and $\rvy \sim q_{Y}$ 
\State Update $h$ by minimizing $L_{h}^{\text{(up)}}$ defined in~\eqref{eq:lr2flow_objective_unpaired}
\State Sample $(\rvx_{k}, \rvy_{k}) \sim q_{X,Y}$
\State Update $h$ by minimizing $L_{h}^{\text{(p)}}$ defined in~\eqref{eq:lr2flow_objective_paired}
\EndWhile
\State Freeze parameters of the mapping $h$
\While{$p_{\text{X$\to$Y}}$ has not converged} \Comment{\textbf{Stage II}: Train probabilistic graphical model}
\State Sample $\rvx \sim q_{X}$ and $\rvy \sim q_{Y}$
\State Compute $\rvh_{\rvx} = h(\rvx)$ and $\rvh_{\rvy} = h(\rvy)$
\State Update $p_{\text{X$\to$Y}}$ by minimizing $L_{\text{g}}^{\text{(up)}}$ defined in~\eqref{eq:loss_graph_unpaired}
\State Sample $(\rvx_{k}, \rvy_{k}) \sim q_{X,Y}$
\State Update $p_{\text{X$\to$Y}}$ by minimizing $L_{\text{g}}^{\text{(p)}}$ defined in~\eqref{eq:sir-nm_paired}
\EndWhile
\State \Return Joint generative model $p_{\text{X$\to$Y}}(\rvx, \rvy)$
\end{algorithmic}
\end{algorithm}

\section{Theoretical Analysis for the MSR mapping}
\label{sec:balance_analysis}
This section provides an analysis of the MSR mapping $h$, which entails an inherent trade-off between the domain consistency~\eqref{eq:msr_domain_consistency} and information preservation~\eqref{eq:msr_information_preservation} conditions. We first demonstrate that these two conditions are 
incompatible.
\begin{proposition}
\label{prop:info_trade-off}
If the conditional entropy satisfies $H(\rvx | \rvy) = -\E_{q_{X,Y}(\rvx, \rvy)}\log q_{X|Y}(\rvx | \rvy) > 0$, 
then~\eqref{eq:msr_domain_consistency} and~\eqref{eq:msr_information_preservation} cannot hold simultaneously .
\end{proposition}

The proof is given in Appendix~\ref{appendix:trade-off_info}. The condition $H(\rvx | \rvy) > 0$ implies that $\rvx$ contains domain-specific stochasticity and cannot be deterministically derived from $\rvy$.
Therefore, to quantitatively analyze the trade-off between these competing objectives, we consider the following optimization problem
\begin{equation}
\label{eq:inference_reconstruction_balance}
\begin{aligned}
\rve_{\gH}^* = & \min_{h\in\gH}\,\E_{q_{X}(\rvx)}\|\rvx - \rvh_{\rvx}\|_{2}^{2}  \\
& \text{s.t.} \, \E_{q_{X,Y}(\rvx,\rvy)}\|q(\rvz | \rvh_{\rvx}) - q(\rvz | \rvh_{\rvy})\|_{1} \leq \varepsilon.
\end{aligned}
\end{equation}
This quantity characterizes the reconstruction capability of $h$ under a relaxed domain invariance (i.e., inference invariance) constraint, where $\varepsilon$ controls the level of allowable deviation.

We evaluate the behavior of the minimal error $\rve_{\gH}^*$ across (i) the noise injection mapping class $\gH_{\text{NI}}$ defined in~\eqref{eq:lud-vae}; (ii) the orthogonal projection mapping class $\gH_{\text{OP}}$ defined in~\eqref{eq:linear_lr2flow}; and (iii) the MSR mapping class $\gH_{\text{MSR}}$ defined in~\eqref{eq:msr_class}. Throughout our analysis, we assume linear Gaussian inference models of the form $q(\rvz | \rvh) = \gN(\rvz | \mA\rvh, {\bm\Gamma})$ for $\rvh \in \{\rvh_{\rvx}, \rvh_{\rvy}\}$, where $\mA\in\sR^{k\times K}$ has full row rank, $\bm{\Gamma}\in\gS^{k}_{++}$, and $k$ is the latent dimension. The following theorem characterizes $\rve^*_{\gH}$ for the hypothesis classes $\gH_{\text{NI}}$ and $\gH_{\text{OP}}$.

\begin{theorem}
\label{thm:balance_ni_op}
(i) Let $\gH_{\text{NI}}$ be defined as in~\eqref{eq:lud-vae}. Then the minimal information preservation error $\rve_{\gH_{\text{NI}}}^*$ defined in~\eqref{eq:inference_reconstruction_balance} satisfies $\rve_{\gH_{\text{NI}}}^* \gtrsim K^2 \, (\mathrm{erf}^{-1}(\frac{\varepsilon}{2}))^{-2}$, where $\mathrm{erf}(z) = \frac{2}{\sqrt{\pi}}\int_{0}^{z}\e^{-t^2}\df t$. In particular, $\rve_{\gH_{\text{NI}}}^* \gtrsim K^2 \, \varepsilon^{-2}$ for sufficiently small $\varepsilon > 0$. \\
(ii) Let $\gH_{\text{OP}}$ be defined as in~\eqref{eq:linear_lr2flow}. Then the minimal information preservation error $\rve_{\gH_{\text{OP}}}^*$ defined in~\eqref{eq:inference_reconstruction_balance} satisfies 
\begin{equation}
\label{eq:bound_balance_linear_lr2flow}
\rve^* \lesssim K - \varepsilon^2.
\end{equation}
\end{theorem}

\begin{proof}[Proof of Theorem~\ref{thm:balance_ni_op}]
\textbf{(i)} Under the given assumptions, the inference model is expressed as $q(\rvz | \rvh_{\rvx}) = \gN(\rvz | \mA\rvx, \sigma^2\mA\mA^{\top} + {\bm\Gamma})$.
Let $\mB_{\sigma} = \sigma^2\mA\mA^{\top} + {\bm\Gamma}$. According to~\cite[Sec. 3.2.1]{fukunaga1990introduction}, we obtain
\begin{equation*}
{\textstyle \|q(\rvz | \rvh_{\rvx}) - q(\rvz | \rvh_{\rvy})\|_{1} = 2\mathrm{erf}(\frac{\|\mB_{\sigma}^{-\frac{1}{2}}\mA(\rvx - \rvy)\|_{2}}{2\sqrt{2}})}.
\end{equation*}
Define $S_{\varepsilon} := \mathrm{erf}^{-1}\left(\frac{\varepsilon}{2}\right)$. Then, the constraint
\begin{equation}
\label{eq:lud-vae_inference_invariance}
\E_{q_{X,Y}(\rvx,\rvy)}\|q(\rvz | \rvh_{\rvx}) - q(\rvz | \rvh_{\rvy})\|_{1} \leq \varepsilon
\end{equation}
implies that $\Tr(\mA^{\top}\mB_{\sigma}^{-1}\mA\mSigma_{\rvn}) \leq 8 S_{\varepsilon}^2$. Observing that the asymptotic equivalence
\begin{equation*}
{\textstyle \Tr(\mA^{\top}\mB_{\sigma}^{-1}\mA\mSigma_{\rvn}) \asymp K/\sigma^{2}}
\end{equation*}
holds in the large-noise regime, where $\asymp$ denotes equivalence up to a constant.
Therefore, a necessary condition to satisfy~\eqref{eq:lud-vae_inference_invariance} is $\sigma^2 \gtrsim {K}/{S_{\varepsilon}^{2}}$. Regarding the expected reconstruction error, we have
\begin{equation*}
\E_{q_{X}(\rvx)}\|\rvx - \rvh_{\rvx}\|_2^2 = K\sigma^2.
\end{equation*}
It follows that $\rve_{\gH_{\text{NI}}}^* \gtrsim {K^2}{S_{\varepsilon}^{-2}}$. In particular, using the expansion $\mathrm{erf}^{-1}(z) = \frac{\sqrt{\pi}}{2}z + o(z)$, we arrive at the bound $\rve_{\gH_{\text{NI}}}^* \gtrsim {K^2}{{\varepsilon}^{-2}}$ for sufficiently small $\varepsilon > 0$.

\noindent \textbf{(ii)} We have $q(\rvz | \rvh_{\rvx}) = \gN(\rvz | \mA\mH\rvx, {\bm\Gamma})$ based on the problem setup, where $\mH \in \gS^{K}$ is an orthogonal projection matrix of rank $d$. Invoking the result from~\cite[Sec. 3.2.1]{fukunaga1990introduction}, we obtain
\begin{equation*}
\begin{aligned}
& \E_{q_{X,Y}(\rvx,\rvy)}\|q(\rvz | \rvh_{\rvx}) - q(\rvz | \rvh_{\rvy})\|_{1}^{2}
\leq \tfrac{2}{\pi}\Tr(\mA^{\top}{\bm\Gamma}^{-1}\mA\mH\mSigma_{\rvn}\mH) 
\leq \kappa \, d,
\end{aligned}
\end{equation*}
where $\kappa = \tfrac{2}{\pi} \, \|{\bm\Gamma}^{-1}\|_{2} \, \|\mA\|_{2}^2 \, \|\mSigma_{\rvn}\|_{2}$. This implies that
\begin{equation}
\label{eq:inference_invariance_bound_linear_lr2flow}
\E_{q_{X,Y}(\rvx,\rvy)}\|q(\rvz | \rvh_{\rvx}) - q(\rvz | \rvh_{\rvy})\|_{1} \leq \sqrt{\kappa d} \lesssim \sqrt{d}.
\end{equation}
Thus, a sufficient condition to guarantee~\eqref{eq:lud-vae_inference_invariance} is $d \lesssim \varepsilon^2$.
To minimize the information preservation error given by
\begin{equation*}
\E_{q_{X}(\rvx)}\|\rvx - \rvh_{\rvx}\|_2^2 = \Tr(\mSigma_{X}(\mI - \mH)),
\end{equation*}
we construct $\mH$ as the orthogonal projection onto the top $d$ principal components of $\mSigma_X$, maximizing the rank $d$ subject to the sufficient condition $d \lesssim \varepsilon^2$. By choosing the maximum valid integer $d \asymp \varepsilon^2$, the minimal reconstruction error achieved by this construction is
\begin{equation*}
\Tr(\mSigma_{X}(\mI - \mH)) \asymp K - d \asymp K - \varepsilon^2.
\end{equation*}
Consequently, the optimal information loss satisfies $\rve_{\gH_{\text{OP}}}^* \lesssim K - \varepsilon^2$.
\end{proof}

Although Theorem~\ref{thm:balance_ni_op} indicates that utilizing an orthogonal projection as a linear surrogate for the MSR mapping theoretically yields a lower information preservation error than noise injection for a fixed inference invariance tolerance, the dependency on the image dimension $K$ in~\eqref{eq:bound_balance_linear_lr2flow} remains a significant source of error. We attribute this limitation to the inherent constraints of linear mappings, which lack the expressive capacity to effectively \emph{disentangle the noise component from the signal distribution}. Instead, we consider the nonlinear MSR mapping class $\gH_{\text{MSR}}$ defined in~\eqref{eq:msr_class}. We establish the following bounds for $\rve_{\gH_{\text{MSR}}}^*$, relying on the expressivity of nonlinear mappings.

\begin{theorem}
\label{thm:balance_nonlinear_lr2flow}
Let $\gH_{\text{MSR}}$ be defined as in~\eqref{eq:msr_class}. Define the intrinsic intersection rank $\rvr^*$ as
\begin{equation}
\label{eq:optim_separate_noise}
\begin{aligned}
& \rvr^* 
= \min \, \mathrm{dim}(\mV^{f}_{X} \cap \mV^{f}_{\rvn}) \\
& \hspace{0.25cm} \mathrm{s.t.} \, \mV^{f}_{X} = \mathrm{Im}\big( \E_{q_{X}(\rvx)}[\rvf_{\rvx} \rvf_{\rvx}^{\top}] \big), \,
\mV^{f}_{\rvn} = \mathrm{Im}\big( \E_{q_{X,Y}(\rvx,\rvy)}[\rvf_{\rvn} \rvf_{\rvn}^{\top}] \big), \\ 
& \hspace{0.8cm} \rvf_{\rvx} = f(\rvx), \, \rvf_{\rvn} = f(\rvy) - f(\rvx), \, f\in\mathrm{Diff}(\sR^{K}).
\end{aligned}
\end{equation}
Then, the minimal information loss $\rve_{\gH_{\text{MSR}}}^*$ defined in~\eqref{eq:inference_reconstruction_balance} satisfies the following conditions: \\
(i) If $\rvr^* = 0$, then $\rve_{\gH_{\text{MSR}}}^* = 0$. \\
(ii) If $\rvr^*> 0$, then $\rve_{\gH_{\text{MSR}}}^* \lesssim \rvr^* - \varepsilon^2$. 
\end{theorem}

\begin{proof}[Proof of Theorem~\ref{thm:balance_nonlinear_lr2flow}]
\textbf{(i)} Consider the case where $\rvr^* = 0$. Let $f^*\in\mathrm{Diff}(\sR^{K})$ be the optimal solution to~\eqref{eq:optim_separate_noise}, and denote the transformed signal and noise components as $\rvf_{\rvx} = f^*(\rvx)$ and $\rvf_{\rvn} = f^*(\rvy) - f^*(\rvx)$, respectively. Let $\mSigma_{X}^* = \E_{q_{X}(\rvx)}[\rvf_{\rvx} \rvf_{\rvx}^{\top}]$, $\mSigma_{\rvn}^* = \E_{q_{X,Y}(\rvx, \rvy)}[\rvf_{\rvn} \rvf_{\rvn}^{\top}]$, and define the image spaces $\mV_{X}^* = \mathrm{Im}(\mSigma_{X}^*)$ and $\mV_{\rvn}^* = \mathrm{Im}(\mSigma_{\rvn}^*)$. By assumption, $\mV^{*}_{X} \cap \mV^*_{\rvn} = \{\mathbf{0}\}$. Defining $h = (f^{*})^{-1} \circ \mathrm{proj}_{\mV^*_{X}} \circ f^*$, where $\mathrm{proj}_{\mV^*_{X}}$ denotes the orthogonal projector onto $\mV^*_{X}$, we bound the squared inference invariance error
\begin{equation*}
\begin{aligned}
\E_{q_{X,Y}(\rvx,\rvy)}
\|q(\rvz | \rvh_{\rvx}) - q(\rvz | \rvh_{\rvy})\|_{1}^{2} 
\lesssim & \E_{q_{X,Y}(\rvx,\rvy)}\|\rvh_{\rvx} - \rvh_{\rvy}\|_{2}^2 \\
\lesssim & \E_{q_{X,Y}(\rvx,\rvy)}
\|\mathrm{proj}_{\mV^*_{X}} (f^*(\rvx) - f^*(\rvy))\|_{2}^2 \\
= & \Tr(\mathrm{proj}_{\mV^*_{X}} \, \mSigma_{\rvn}^*) = 0.
\end{aligned}
\end{equation*}
Simultaneously, the information preservation error satisfies
\begin{equation*}
\begin{aligned}
\E_{q_{X}(\rvx)}\|\rvx - \rvh_{\rvx}\|_{2}^2 
\lesssim & \E_{q_{X}(\rvx)}\|(\mI - \mathrm{proj}_{\mV^*_{X}}) \, f^*(\rvx)\|_{2}^2 \\
= & \Tr((\mI - \mathrm{proj}_{\mV^*_{X}}) \, \mSigma_{X}^*) = 0.
\end{aligned}
\end{equation*}
This demonstrates $\rve^*_{\gH_{\text{MSR}}} = 0$.

\noindent \textbf{(ii)} For the case $\rvr^* > 0$, we first consider the simplified scenario where $f$ in~\eqref{eq:optim_separate_noise} is the identity map and $h$ is a linear projection $h(\rvx) = \mH \rvx$, with $\mH\in\gS^{K}_{+}$ and $\mH^2=\mH$. Let $\sR^{K}$ be decomposed as $\sR^{K} = \bigoplus_{i=1}^{3}\mV_{i}$, such that $\mathrm{Im}(\mSigma_{X}) = \mV_{1} \oplus \mV_{2}$ and $\mathrm{Im}(\mSigma_{\rvn}) = \mV_{2} \oplus \mV_{3}$. Let $r_{i} = \mathrm{dim}(\mV_{i})$. Without loss of generality, we assume $\mV_{1} = \mathrm{span}(\{\rve_i\}_{i=1}^{r_1})$, $\mV_{2} = \mathrm{span}(\{\rve_i\}_{i=r_1+1}^{r_1+r_2})$, and $\mV_{3} = \mathrm{span}(\{\rve_i\}_{i=r_1+r_2+1}^{K})$, where $\rve_{i}$ is the $i$-th standard basis vector. We construct the block diagonal matrix $\mH$ as
\begin{equation}
\label{eq:block_H}
\mH = \mathrm{diag} \left(\mI_{r_1+d},  \mathbf{0}_{r_2+r_3 - d}\right).
\end{equation}
As established in~\eqref{eq:inference_invariance_bound_linear_lr2flow}, a sufficient condition to satisfy~\eqref{eq:lud-vae_inference_invariance}
is $d \lesssim \varepsilon^2$. With the structure of $\mH$ defined in~\eqref{eq:block_H}, the information preservation error becomes
\begin{equation}
\label{eq:information_preservation_bound_linear_lr2flow}
\E_{q_{X}(\rvx)}\|\rvx - \rvh_{\rvx}\|_2^2 = \Tr(\mSigma_{X}(\mI - \mH)) \asymp r_2 - d \asymp r_2 - \varepsilon^2.
\end{equation}
For the general case, let
$f^*\in\mathrm{Diff}(\sR^{K})$ be the optimal solution
to~\eqref{eq:optim_separate_noise}.
We derive a result analogous to~\eqref{eq:information_preservation_bound_linear_lr2flow} by analyzing the covariance matrices in the transformed space:
$\mSigma^{*}_{X} = \E_{q_{X}(\rvx)}[\rvf_{\rvx} \rvf_{\rvx}^{\top}]$
and
$\mSigma_{\rvn}^{*} = \E_{q_{X,Y}(\rvx, \rvy)}[\rvf_{\rvn} \rvf_{\rvn}^{\top}]$.
In this setting, we define
$h = (f^*)^{-1} \circ \mH \circ f^*$, where $\mH$ follows the
structure in~\eqref{eq:block_H}. Consequently, the rank term $r_2$
in~\eqref{eq:information_preservation_bound_linear_lr2flow}
corresponds directly to the intrinsic intersection rank $\rvr^*$
defined in~\eqref{eq:optim_separate_noise}, yielding the final bound
$\rve_{\text{MSR}}^* \lesssim \rvr^* - \varepsilon^2$.
\end{proof}

Theorem~\ref{thm:balance_nonlinear_lr2flow} characterizes the information preservation error via the intrinsic separability of the signal and noise components. Specifically, it demonstrates that perfect inference invariance $q(\rvz | \rvh_{\rvx}) = q(\rvz | \rvh_{\rvy})$ is achievable with zero information loss, provided the underlying signal and noise manifolds are perfectly separable.


\section{Experiments}
This section comprehensively evaluates the practical performance of LUD-MSR, encompassing real-world noise modeling and downstream denoising applications, with a specific focus on the challenging cryo-EM image denoising task. The highest quantitative scores are highlighted in \textbf{bold}.

\subsection{Results for the Generated Noisy Images}
\label{subsec:noise_synthesis}

In this part, we evaluate the fidelity of images synthesized via the clean-to-noisy generation pipeline of LUD-MSR, i.e., $p_{\text{X$\to$Y}}$ defined in~\eqref{eq:lud-msr_joint}.

\noindent\textbf{Experimental settings.}
LUD-MSR is trained on the SIDD-Medium dataset~\cite{abdelhamed2018sidd}, which comprises 320 clean-noisy image pairs captured across 5 camera sensors and 15 ISO levels. We generate training data by randomly cropping 300 patches of size $256\times 256$ per image, yielding 96,000 paired samples. To simulate the unpaired setting, we partition these pairs into two disjoint subsets of 48,000 samples each, using clean images from the first subset and noisy images from the second. For semi-supervised experiments, we sample random subsets of 0.01\% (10 pairs), 0.1\% (96 pairs), and 1\% (960 pairs) of the total data to serve as the paired domain.

\noindent\textbf{Training details.} 
LUD-MSR trains in two stages, as outlined in Algorithm~\ref{alg:sir-nm_train} and Algorithm~\ref{alg:sir-nm_semi_train}. 
In Stage I, we train the MSR mapping for 20k, 50k, 65k, and 200k iterations for the unpaired, 10-pair, 96-pair, and 960-pair settings, respectively. We optimize using AdamW~\cite{loshchilov2017decoupled} with an initial learning rate of $2\times 10^{-4}$, halved every 50k iterations. We use a batch size of 8 and $128\times 128$ random crops with flips and rotations. The Gaussian noise level $\sigma$ in~\eqref{eq:lr2flow_objective_unpaired} is sampled uniformly from $[0, \tfrac{40}{255}]$.
In Stage II, we optimize the probabilistic graphical model for 500k iterations using Adam~\cite{kingma2014adam}. The initial learning rate of $10^{-4}$ is halved at 100k, 175k, 250k, 325k, 400k, and 450k iterations. We use a batch size of 8, $64\times 64$ random crops with augmentations, and empirically set the mixture weight to $\alpha=0.5$ in~\eqref{eq:inference_models_unpaired} and~\eqref{eq:inference_model_paired}.

\noindent\textbf{Compared methods.}
Given that LUD-MSR is capable of synthesizing realistic noisy images in both unpaired and semi-supervised scenarios, we compare LUD-MSR against (i) unpaired noise modeling methods~\cite{jang2021c2n, wolf2021deflow, zheng2022learn} and (ii) semi-supervised approaches~\cite{zheng2024senm}. We also include supervised noise models~\cite{yue2020dual, kousha2022modeling, fu2023srgb, zheng2024senm, kim2024srgb} as a reference. Synthesized noise quality is assessed via Average KL Divergence (AKLD)~\cite{yue2020dual} and Fr\'echet Inception Distance (FID)~\cite{heusel2017gans}.

\begin{table}[htbp]
\centering
\caption{Quantitative Evaluation of Noise Generation Performance on Real-World Noise Datasets}
\label{tab:quant_synthetic_noise}
{
\setlength{\tabcolsep}{3pt}
\footnotesize
{%
\begin{tabular}{@{} l c c c c c @{}}
\toprule
\multirow{2}{*}{Model} & \multirow{2}{*}{\# Paired Data} & \multicolumn{2}{c}{SIDD validation} & \multicolumn{2}{c}{SIDD+} \\
\cmidrule(lr){3-4} \cmidrule(l){5-6} 
& & AKLD & FID & AKLD & FID \\
\midrule
C2N~\cite{jang2021c2n}        & \multirow{4}{*}{0}           & 0.2802          & 34.04          & 0.8252          & 38.97          \\
DeFlow~\cite{wolf2021deflow}  &                              & --              & 39.45          & --              & --             \\
LUD-VAE~\cite{zheng2022learn} &                              & 0.2177          & 35.72          & 0.2497          & 69.67          \\
LUD-MSR                       &                              & \textbf{0.1111} & \textbf{18.05} & \textbf{0.1459} & \textbf{38.24} \\
\midrule
SeNM-VAE~\cite{zheng2024senm} & \multirow{2}{*}{0.01\% (10)} & 0.1282          & 17.25          & 0.2113          & 45.23          \\
LUD-MSR                       &                              & \textbf{0.1099} & \textbf{16.19} & \textbf{0.1471} & \textbf{35.76} \\
\midrule
SeNM-VAE~\cite{zheng2024senm} & \multirow{2}{*}{0.1\% (96)}  & 0.1162          & 16.76          & 0.1649          & 37.21          \\
LUD-MSR                       &                              & \textbf{0.1097} & \textbf{15.09} & \textbf{0.1505} & \textbf{34.96} \\
\midrule
SeNM-VAE~\cite{zheng2024senm} & \multirow{2}{*}{1\% (960)}   & 0.1166          & 15.10          & 0.1717          & 40.88          \\
LUD-MSR                       &                              & \textbf{0.1083} & \textbf{14.35} & \textbf{0.1507} & \textbf{34.39} \\
\midrule
DANet~\cite{yue2020dual}      & \multirow{6}{*}{100\%}       & 0.2124          & 26.25          & 0.4153          & 72.95          \\
Flow-sRGB~\cite{kousha2022modeling} &                        & 0.4716          & 28.60          & 0.5855          & 45.14          \\
NeCA-S~\cite{fu2023srgb}      &                              & 1.0470          & 36.43          & 1.7921          & 56.36          \\
NeCA-W~\cite{fu2023srgb}      &                              & 0.1436          & 19.96          & 0.2408          & 29.27          \\
NAFlow~\cite{kim2024srgb}     &                              & 0.1306          & 18.31          & 0.2911          & 39.21          \\
SeNM-VAE~\cite{zheng2024senm} &                              & 0.1180          & 13.79          & 0.1626          & 37.78          \\
\bottomrule
\end{tabular}
}
}
\end{table}

\noindent\textbf{Results.}
Table~\ref{tab:quant_synthetic_noise} summarizes the quantitative noise synthesis results on the SIDD validation set and the SIDD+ dataset~\cite{abdelhamed2020ntire}. LUD-MSR significantly outperforms state-of-the-art unpaired and semi-supervised baselines, achieving performance that is highly competitive with fully supervised models. This demonstrates the superiority of our probabilistic graphical model design and the inductive bias introduced by the MSR mapping~\eqref{eq:msr_mapping}. A qualitative comparison in Figure~\ref{fig:quant_noise_gen} confirms that LUD-MSR generates markedly more realistic noise patterns than the unpaired baselines, where C2N produces over-smoothed textures, and LUD-VAE introduces noticeable artifacts from excessive information loss. Furthermore, LUD-MSR exhibits robust generalization on the SIDD+ dataset, effectively capturing noise conditioned on scenes and lighting conditions unseen in the SIDD-Medium training set.

\begin{figure*}
\centering
\includegraphics[width=\linewidth]{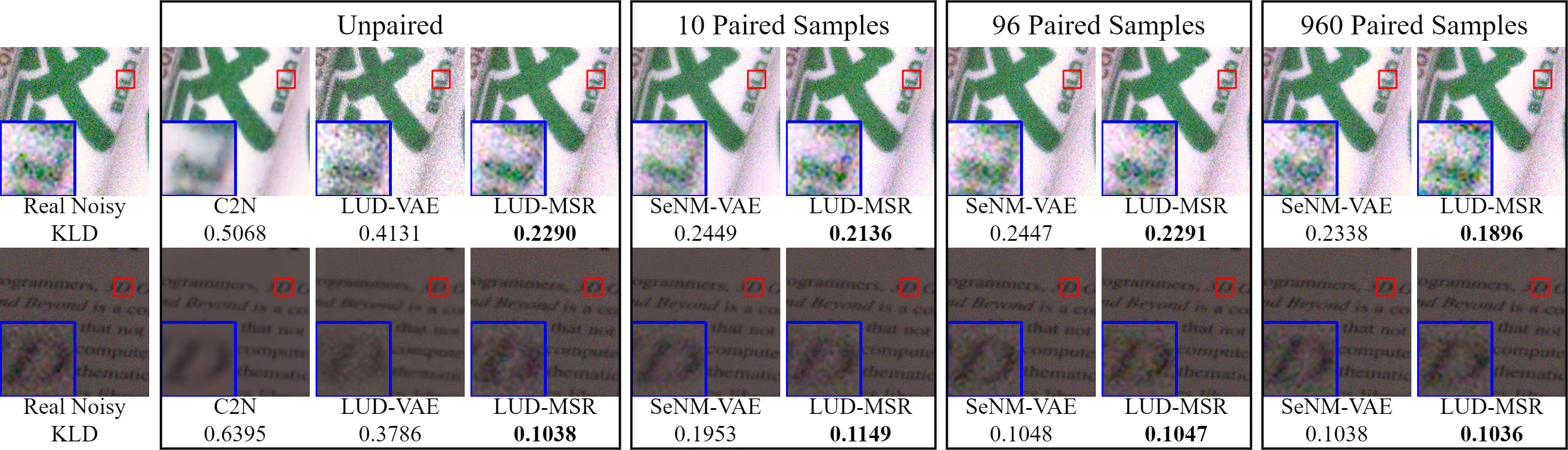}
\caption{Visual comparison of noisy images generated by various noise modeling methods on the SIDD validation set. KLD ($\downarrow$) is computed following~\cite{yue2020dual} with $L=1$.}
\label{fig:quant_noise_gen}
\end{figure*}

\subsection{Results for Downstream task: Real-World Image Denoising}
\label{subsec:downstream_denoising}
In this part, We evaluate the downstream performance of LUD-MSR on real-world denoising benchmarks.

\noindent\textbf{Experimental settings.}
Using trained LUD-MSR, we generate a pseudo-paired training dataset by sampling clean images from the SIDD-Medium set and synthesizing their noisy counterparts via the clean-to-noisy generation pipeline.
We train a DnCNN~\cite{zhang2017beyond} denoiser on this synthetic dataset and evaluate it across real-world benchmarks, including SIDD validation set, SIDD benchmark, DND benchmark~\cite{plotz2017dnd}, and the SIDD+, PolyU~\cite{xu2018polyu}, CC~\cite{nam2016cc} datasets. Denoising quality is assessed via Peak Signal-to-Noise Ratio (PSNR) and Structural Similarity Index (SSIM).

\noindent\textbf{Training details.} 
Following the protocols in~\cite{zheng2022learn,zheng2024senm}, we train DnCNN for 300k iterations using Adam. The initial learning rate of $10^{-4}$ is halved every 100k iterations. We use a batch size of 64 and $40 \times 40$ random crops with flips and rotations. We validate denoising performance on the SIDD validation set every 5k iterations, selecting the checkpoint with the highest PSNR for final evaluation.

\noindent\textbf{Compared methods.} 
We compare LUD-MSR against representative noise modeling methods~\cite{jang2021c2n,zheng2022learn,zheng2024senm,yue2020dual,kousha2022modeling,fu2023srgb,kim2024srgb}. We also include a DnCNN model trained on real paired data as a reference.

\noindent\textbf{Results.} 
Table~\ref{tab:downstream_denoising_quantitative} details the quantitative denoising performance, where LUD-MSR outperforms competing methods by achieving the highest PSNR and SSIM scores across most benchmarks. Qualitative comparisons in Figure~\ref{fig:denoising_quant} confirm that LUD-MSR restores cleaner images with sharper details. We attribute these gains to our joint distribution modeling framework, which captures diverse, realistic noise characteristics to effectively bridge the domain gap between synthetic training data and real-world distributions.

\begin{figure*}
\centering
\includegraphics[width=\linewidth]{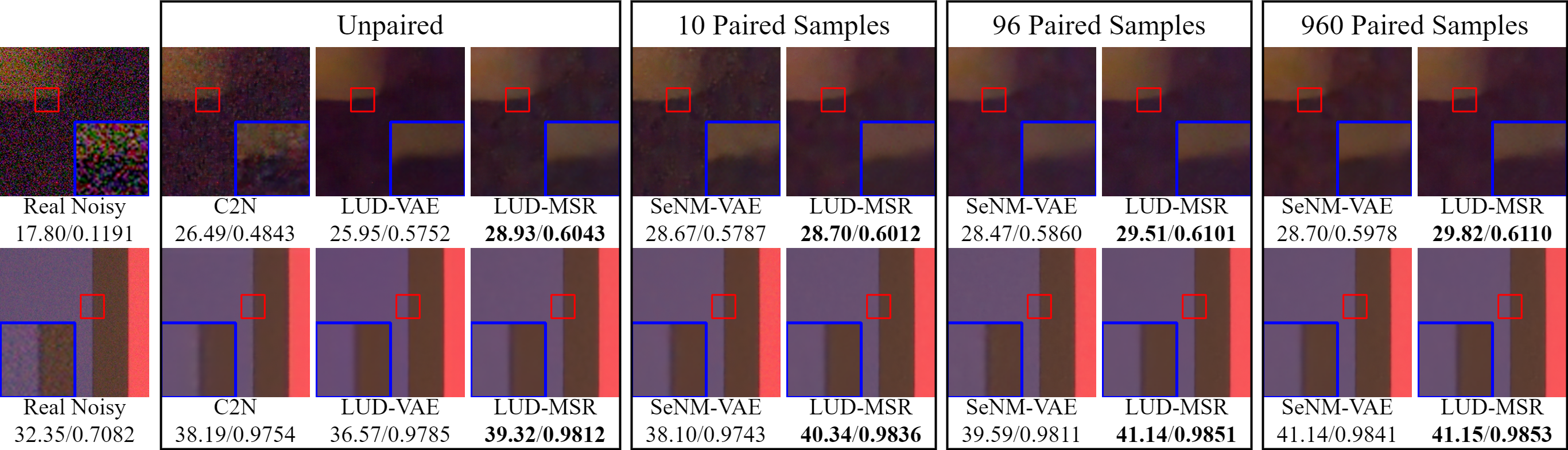}
\caption{Visual comparison of denoising results for DnCNN models trained on synthetic data from various noise modeling methods. PSNR/SSIM are reported.}
\label{fig:denoising_quant}
\end{figure*}

\begin{table*}[t]
\centering
\caption{Quantitative Denoising Performance of DnCNN Models Trained on Synthetic Data from Various Noise Modeling Methods}
\label{tab:downstream_denoising_quantitative}
{
\setlength{\tabcolsep}{3pt}
\footnotesize
\resizebox{\textwidth}{!}{%
\begin{tabular}{@{} l c *{12}{c} @{}}
\toprule
\multirow{2}{*}{Methods} & \multirow{2}{*}{\# Paired Data} & \multicolumn{2}{c}{SIDD validation} & \multicolumn{2}{c}{SIDD benchmark} & \multicolumn{2}{c}{DND} & \multicolumn{2}{c}{SIDD+} & \multicolumn{2}{c}{PolyU} & \multicolumn{2}{c@{}}{CC} \\
\cmidrule(lr){3-4} \cmidrule(lr){5-6} \cmidrule(lr){7-8} \cmidrule(lr){9-10} \cmidrule(lr){11-12} \cmidrule(l){13-14}
& & PSNR & SSIM & PSNR & SSIM & PSNR & SSIM & PSNR & SSIM & PSNR & SSIM & PSNR & SSIM \\
\midrule
C2N~\cite{jang2021c2n} & \multirow{4}{*}{0} & 34.11 & 0.818 & 35.34 & 0.806 & 36.08 & 0.903 & 34.27 & 0.852 & 37.07 & 0.955 & 35.24 & \textbf{0.938} \\
DeFlow~\cite{wolf2021deflow} & & 33.82 & -- & -- & -- & 36.71 & 0.923 & -- & -- & -- & -- & -- & -- \\
LUD-VAE~\cite{zheng2022learn} & & 34.91 & 0.892 & 35.49 & 0.883 & 37.60 & 0.933 & 34.52 & 0.891 & 33.68 & 0.942 & 31.25 & 0.920 \\
LUD-MSR & & \textbf{36.81} & \textbf{0.899} & \textbf{37.20} & \textbf{0.890} & \textbf{38.54} & \textbf{0.942} & \textbf{34.91} & \textbf{0.892} & \textbf{37.44} & \textbf{0.956} & \textbf{36.01} & 0.936 \\
\midrule
SeNM-VAE~\cite{zheng2024senm} & \multirow{2}{*}{0.01\% (10)} & 36.73 & 0.891 & 37.12 & 0.881 & 37.94 & 0.936 & 35.42 & 0.898 & 37.16 & 0.954 & 35.37 & 0.926 \\
LUD-MSR & & \textbf{37.14} & \textbf{0.902} & \textbf{37.52} & \textbf{0.894} & \textbf{38.31} & \textbf{0.942} & \textbf{35.96} & \textbf{0.904} & \textbf{37.47} & \textbf{0.957} & \textbf{36.36} & \textbf{0.950} \\
\midrule
SeNM-VAE~\cite{zheng2024senm} & \multirow{2}{*}{0.1\% (96)} & 36.92 & 0.895 & 37.32 & 0.886 & 38.21 & 0.942 & 35.09 & 0.860 & 35.60 & 0.938 & 34.32 & 0.906 \\
LUD-MSR & & \textbf{37.42} & \textbf{0.901} & \textbf{37.72} & \textbf{0.892} & \textbf{38.78} & \textbf{0.947} & \textbf{35.79} & \textbf{0.901} & \textbf{37.50} & \textbf{0.954} & \textbf{35.75} & \textbf{0.928} \\
\midrule
SeNM-VAE~\cite{zheng2024senm} & \multirow{2}{*}{1\% (960)} & 37.28 & 0.902 & 37.58 & 0.895 & 38.44 & 0.943 & \textbf{35.46} & 0.889 & 37.34 & 0.953 & 35.40 & 0.918 \\
LUD-MSR & & \textbf{37.68} & \textbf{0.903} & \textbf{37.95} & \textbf{0.895} & \textbf{38.87} & \textbf{0.947} & 35.44 & \textbf{0.891} & \textbf{37.39} & \textbf{0.953} & \textbf{35.63} & \textbf{0.924} \\
\midrule
DANet~\cite{yue2020dual} & \multirow{5}{*}{100\%} & 36.25 & 0.885 & 36.63 & 0.875 & 38.21 & 0.943 & 34.82 & 0.882 & 37.17 & 0.958 & 33.86 & 0.928 \\
Flow-sRGB~\cite{kousha2022modeling} & & 33.24 & 0.868 & 34.41 & 0.857 & 36.09 & 0.895 & 34.12 & 0.872 & 36.12 & 0.953 & 33.46 & 0.920 \\
NeCA-W~\cite{fu2023srgb} & & 37.04 & 0.900 & 37.35 & 0.892 & 38.70 & 0.946 & 35.65 & 0.901 & 37.25 & 0.958 & 36.23 & 0.949 \\
NAFlow~\cite{kim2024srgb} & & 37.23 & 0.898 & 37.61 & 0.891 & 38.53 & 0.943 & 36.63 & 0.924 & 36.59 & 0.951 & 34.94 & 0.937 \\
SeNM-VAE~\cite{zheng2024senm} & & 38.28 & 0.907 & 38.55 & 0.900 & 39.09 & 0.950 & 35.67 & 0.898 & 36.46 & 0.931 & 34.21 & 0.881 \\
\midrule
Real Noise & 100\% & 38.40 & 0.909 & 38.65 & 0.901 & 39.04 & 0.949 & 35.94 & 0.903 & 36.61 & 0.930 & 34.50 & 0.889 \\
\bottomrule
\end{tabular}
}
}
\end{table*}

\subsection{Application in Cryo-EM Denoising}
\label{subsec:experiments_cryoEM}
We apply LUD-MSR to the challenging task of cryo-EM image denoising. Specifically, we employ LUD-MSR to model the complex, signal-dependent noise distributions inherent to real-world cryo-EM micrographs. We then train a denoising network on the generated pseudo-paired dataset to recover high-fidelity biological structures from low-SNR observations.

\noindent\textbf{Experimental settings.}
We conduct experiments on three real-world cryo-EM datasets: EMPIAR-10025~\cite{campbell20152}, EMPIAR-10028~\cite{10.7554/eLife.03080}, and EMPIAR-10077~\cite{Fischer2016ThePT}. To construct the unpaired training data, we select reference proteins from the protein data bank (PDB) that are structurally homologous to the target macromolecules. Leveraging sequence similarity, these homologues typically share low-frequency structural features with the target proteins. Specifically, we utilize structures 6HVR (for 10025), 7ZJW (for 10028), and 8UTJ (for 10077) as references. The atomic coordinates of these homologous proteins are first converted into 3D density maps at the target resolution. Subsequently, we project these density maps from random orientations to generate noise-free simulated projections. To faithfully mimic the physical imaging process, we apply contrast transfer function (CTF) modulation matched to the intensity profile of the real data to these projections, thereby establishing the pseudo-clean domain. We assess downstream performance by analyzing the quality of the denoised 2D particle class averages. For quantitative comparison against other cryo-EM denoising methods, we employ the estimated Signal-to-Noise Ratio (SNR)~\cite{li2022noise} as our primary metric.

\noindent\textbf{Training details.}
Our pipeline trains the LUD-MSR and a subsequent denoising network using identical configurations across all datasets. For LUD-MSR, we optimize the MSR mapping using AdamW for 300k iterations with a batch size of 4. The initial learning rate of $2\times 10^{-4}$ is halved at 50k, 100k, 150k, and 200k iterations. The training objective follows~\eqref{eq:lr2flow_objective_unpaired}, where the Gaussian noise level $\sigma$ is uniformly sampled from $[0, 0.99]$. The probabilistic graphical model is optimized via Adam for 300k iterations with a batch size of 2, starting with a $10^{-4}$ learning rate halved at 75k, 150k, 200k, and 250k iterations. For downstream denoising, we train the DRUNet~\cite{zhang2021plug} architecture for 300k iterations using Adam. We use a batch size of 8 and an initial learning rate of $10^{-4}$, halved every 50k iterations.




\noindent\textbf{Compared methods.} 
We compare LUD-MSR against prominent model-based approaches, including the low-pass filter, the fast Gaussian binomial filter~\cite{haddad1991class}, and BM3D~\cite{dabov2007image}, alongside the state-of-the-art learning-based method Topaz~\cite{bepler2020topaz}.

\begin{table}[htbp]
\centering
\caption{Quantitative Comparison of Cryo-EM Image Denoising Methods Based on Estimated SNR~\cite{li2022noise}.}
\label{tab:quant_cryoEM}
{
\setlength{\tabcolsep}{4pt}
\footnotesize
{%
\begin{tabular}{@{} l *{5}{c} @{}}
\toprule
Dataset & Low-Pass & Gaussian~\cite{haddad1991class} & BM3D~\cite{dabov2007image} & Topaz~\cite{bepler2020topaz} & LUD-MSR \\
\midrule
10025 & 1.94 & -0.68 & 11.69 & -4.73 & \textbf{14.54} \\
10028 & 5.62 & -0.09 & 4.89 & -0.75 & \textbf{13.62} \\
10077 & 0.36 & -0.92 & 4.41 & -0.75 & \textbf{13.08} \\
\bottomrule
\end{tabular}
}
}
\end{table}

\noindent\textbf{Results.} 
Table~\ref{tab:quant_cryoEM} presents the estimated SNR results for the evaluated cryo-EM denoising approaches. LUD-MSR significantly surpasses existing baselines. Visual comparisons in Figure~\ref{fig:cryo_denoise} confirm that while competing methods struggle to distinguish signal from background noise under extremely low-SNR conditions, LUD-MSR successfully restores particles with high-fidelity structural details and fine textures. This substantial improvement underscores LUD-MSR's capacity to capture complex real-world noise characteristics, demonstrating its potential for challenging scientific applications.

\begin{figure*}
\centering
\includegraphics[width=\linewidth]{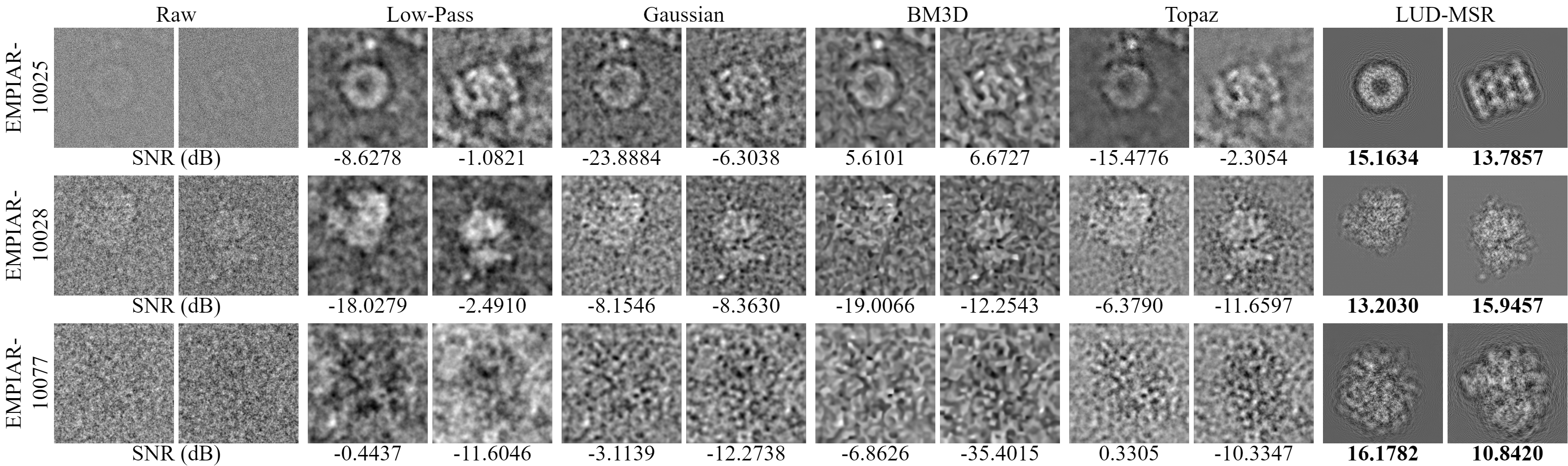}
\caption{Visual comparison of denoising results across three real-world cryo-EM datasets: EMPIAR-10025, EMPIAR-10028, and EMPIAR-10077.}
\label{fig:cryo_denoise}
\end{figure*}

\subsection{Ablation Study}
Our default MSR mapping~\eqref{eq:msr_mapping} employs $T=3$ scale levels and a linear B-spline wavelet tight frame for the wavelet module $\mW$. In this part, we ablate these two configurations and report the resulting PSNR and SSIM metrics on the SIDD validation set.

\noindent\textbf{Scale levels $T$.}
While keeping the wavelet module $\mW$ fixed as a linear B-spline wavelet tight frame, we evaluate the impact of $T$ on noise synthesis fidelity and downstream denoising performance. As shown in Table~\ref{tab:ablate_scale_level}, $T=3$ yields the best overall performance, indicating an optimal trade-off between domain consistency and information preservation.

\begin{table}
\centering
\caption{Ablation Study on the Scale Levels $T$ within the MSR Design.}
\label{tab:ablate_scale_level}
{
\footnotesize
\setlength{\tabcolsep}{3pt}
{%
\begin{tabular}{@{}ccccccc@{}}
\toprule
\multirow{2}{*}{$T$} & \multicolumn{2}{c}{MSR Performance} & \multicolumn{2}{c}{Noise Gen.} & \multicolumn{2}{c}{Denoising} \\
\cmidrule(lr){2-3} \cmidrule(lr){4-5} \cmidrule(lr){6-7}
& Rec. PSNR & Align. PSNR & FID & AKLD & PSNR & SSIM \\
\midrule
1 & \textbf{45.45} & 36.27 & 24.99 & 0.1276 & 36.24 & 0.8918 \\
2 & 41.77 & 38.02 & 22.94 & \textbf{0.1083} & 36.67 & \textbf{0.8993} \\
3 & 39.41 & 40.22 & \textbf{18.05} & 0.1111 & \textbf{36.81} & 0.8991 \\
4 & 37.71 & \textbf{41.27} & 22.04 & 0.1213 & 36.67 & 0.8949 \\
\bottomrule
\end{tabular}
}
}
\end{table}

\noindent\textbf{Wavelet module $\mW$.}
We compare the use of a (linear B-spline) wavelet tight frame for $\mW$ against the Haar wavelet. Furthermore, interpreting our MSR mapping~\eqref{eq:msr_mapping} as a learnable resizing operator, we evaluate non-learnable interpolation baselines formulated as $h = \gI_{s}\circ \gI_{1/s}$, where $\gI_{\tau}$ denotes Bicubic or Bilinear resizing with a scale factor of $\tau$. To ensure a fair comparison of feature granularity, we set $T=3$ for the learnable mappings and $s=8$ for the non-learnable baselines. Table~\ref{tab:ablate_msr_implementation} demonstrates that the wavelet tight frame achieves the best noise generation quality and downstream denoising performance. This result confirms the advantages of our data-adaptive design strategy, while the performance gain of the wavelet tight frame over the Haar wavelet highlights the benefits of its representation flexibility.

\begin{table}[!t]
\centering
\caption{Ablation Study on the wavelet module. \enquote{TF} Refers to \enquote{Wavelet Tight Frame}.}
\label{tab:ablate_msr_implementation}
{
\footnotesize
\setlength{\tabcolsep}{3pt}
{%
\begin{tabular}{@{}lcccccc@{}}
\toprule
\multirow{2}{*}{\makecell[l]{MSR \\ Impl.}} & \multicolumn{2}{c}{MSR Performance} & \multicolumn{2}{c}{Noise Gen.} & \multicolumn{2}{c}{Denoising} \\
\cmidrule(lr){2-3} \cmidrule(lr){4-5} \cmidrule(lr){6-7}
& Rec. PSNR & Align. PSNR & FID & AKLD & PSNR & SSIM \\
\midrule
Bilinear & 36.90 & 36.79 & 33.55 & 0.1465 & 35.22 & 0.8870 \\
Bicubic  & 38.14 & 35.07 & 32.66 & 0.1231 & 35.47 & 0.8846 \\
Haar     & \textbf{39.43} & 38.00 & 18.89 & \textbf{0.1093} & 35.97 & 0.8959 \\
TF       & 39.41 & \textbf{40.22} & \textbf{18.05} & 0.1111 & \textbf{36.81} & \textbf{0.8991} \\
\bottomrule
\end{tabular}
}
}
\end{table}

\section{Conclusion}
We propose LUD-MSR, a novel unpaired joint distribution modeling framework based on a probabilistic graphical model utilizing two auxiliary variables. For image noise modeling, we explicitly construct these variables via Multi-Scale image Representation (MSR) mapping. Theoretical analysis establishes that LUD-MSR achieves a superior trade-off between domain consistency
and information preservation compared to existing approaches, thus significantly reducing joint distribution approximation error. Through its clean-to-noisy generation pipeline, LUD-MSR synthesizes high-fidelity paired data, effectively boosting downstream denoising performance. Furthermore, LUD-MSR demonstrates exceptional versatility on the challenging task of cryo-EM image denoising.

\appendix

\section{Architecture of Probabilistic Graphical Model}
\label{appendix:architecture}
\noindent\textbf{Hierarchical structure.}
To enhance generative performance, we adopt a hierarchical VAE~\cite{child2020very,vahdat2020nvae} architecture for the probabilistic graphical model of LUD-MSR. Specifically, we represent the latent variables $\rvz$ and $\rvz_{\rvn}$ as sequences across $L=7$ layers
\begin{equation*}
\rvz = \left(\rvz^1,\rvz^2,\ldots,\rvz^L\right), \quad \rvz_{\rvn} = \left(\rvz_{\rvn}^1, \rvz_{\rvn}^2,\ldots,\rvz_{\rvn}^L\right).
\end{equation*}
Applying the chain rule of probability, we obtain the following autoregressive decompositions
\begin{equation*}
\begin{aligned}
& {\textstyle q(\rvz | \rvx) = \prod_{l=1}^{L}q(\rvz^{l} | \rvz^{>l}, \rvx) \text{, } q(\rvz | \rvh_{\rvx}) = \prod_{l=1}^{L}q(\rvz^{l} | \rvz^{>l}, \rvh_{\rvx})}, \\
& {\textstyle q(\rvz | \rvy) = \prod_{l=1}^{L}q(\rvz^{l} | \rvz^{>l}, \rvy) \text{, } q(\rvz | \rvh_{\rvy}) = \prod_{l=1}^{L}q(\rvz^{l} | \rvz^{>l}, \rvh_{\rvy})},
\end{aligned}
\end{equation*}
where $\rvz^{>l} = (\rvz^{l+1}, \ldots, \rvz^{L})$. Furthermore, we define
\begin{equation*}
\begin{aligned}
p(\rvz_{\rvn} | \rvz) = & {\textstyle \prod_{l=1}^{L} p(\rvz_{\rvn}^{l} | \rvz_{\rvn}^{>l}, \rvz^{\geq l})}, \\
q(\rvz_{\rvn} | \rvy, \rvz) = & {\textstyle \prod_{l=1}^{L} q(\rvz_{\rvn}^{l} | \rvy, \rvz_{\rvn}^{>l}, \rvz^{\geq l})},
\end{aligned}
\end{equation*}
where $\rvz^{\geq l}$ and $\rvz_{\rvn}^{>l}$ are defined analogously to $\rvz^{>l}$.
Consequently, the full KL divergence terms $\KL(q(\rvz | \rvx) || q(\rvz | \rvh_{\rvx}))$, $\KL(q(\rvz | \rvy) || q(\rvz | \rvh_{\rvy}))$, and $\KL(q(\rvz_{\rvn} | \rvy, \rvz) || p(\rvz_{\rvn} | \rvz))$ can be decomposed layer-wise as
\begin{equation*}
\begin{aligned}
& {\textstyle \KL(q(\rvz | \rvx) || q(\rvz | \rvh_{\rvx})) = \sum_{l=1}^{L}\KL(q(\rvz^{l} | \rvx, \rvz^{>l}) || q(\rvz^{l} | \rvh_{\rvx}, \rvz^{>l}))}, \\
& {\textstyle \KL(q(\rvz | \rvy) || q(\rvz | \rvh_{\rvy})) = \sum_{l=1}^{L}\KL(q(\rvz^{l} | \rvy, \rvz^{>l}) || q(\rvz^{l} | \rvh_{\rvy}, \rvz^{>l}))} ,\\
& {\textstyle \KL(q(\rvz_{\rvn} | \rvy, \rvz) || p(\rvz_{\rvn} | \rvz)) = \sum_{l=1}^{L}\KL(q(\rvz_{\rvn}^{l} | \rvy, \rvz^{\geq l}, \rvz_{\rvn}^{>l}) || p(\rvz_{\rvn}^{l} | \rvz^{\geq l}, \rvz_{\rvn}^{>l}))} .
\end{aligned}
\end{equation*}
The distributions $q(\rvz^l | \rvx, \rvz^{>l})$, $q(\rvz^l | \rvh_{\rvx}, \rvz^{>l})$, $q(\rvz^l | \rvy,\rvz^{>l})$, $q(\rvz^l | \rvh_{\rvy},\rvz^{>l})$, $p(\rvz_{\rvn}^l | \rvz^{\geq l}, \rvz_{\rvn}^{>l})$, and $q(\rvz_{\rvn}^l | \rvy, \rvz^{\geq l}, \rvz_{\rvn}^{>l})$ are modeled as parameterized Gaussian distributions, facilitating the analytical computation of the constituent KL divergences.

\noindent\textbf{Model architecture.} 
\begin{figure*}
\centering
\includegraphics[width=\linewidth]{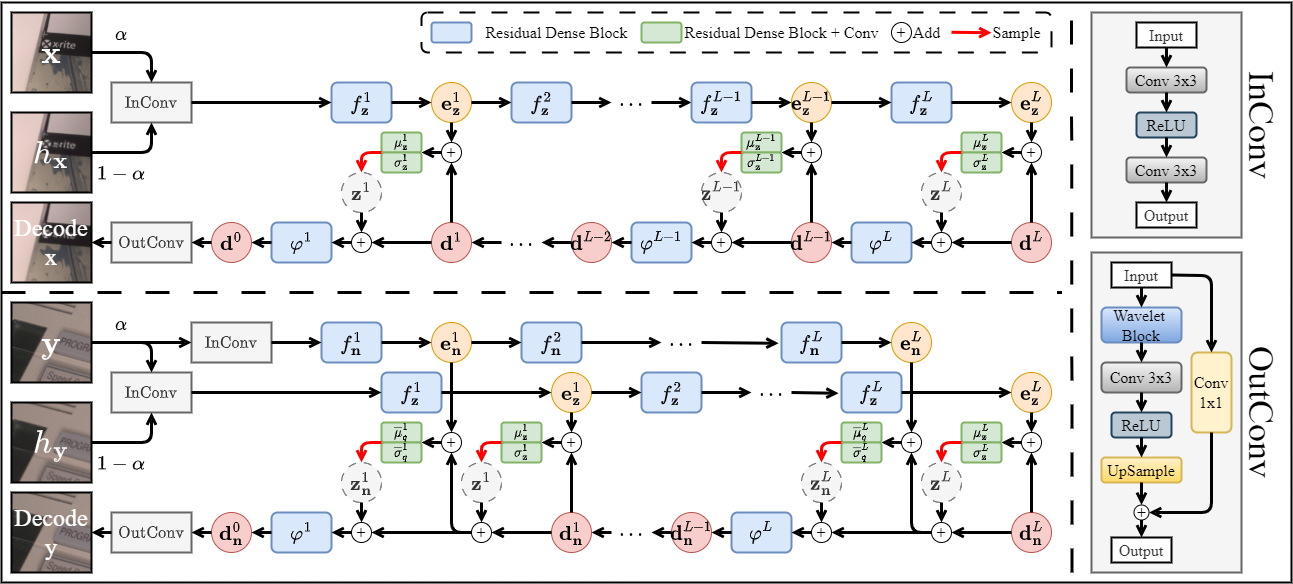}
\caption{Hierarchical architecture of the probabilistic graphical model of LUD-MSR.}
\label{fig:arch_nm}
\end{figure*}
Our hierarchical architecture is illustrated in Figure~\ref{fig:arch_nm}. For the $\rvx$-domain branch, we define
\begin{equation}
\label{eq:inference_z}
{\textstyle q(\rvz^l | \rvx, \rvz^{>l}) = \gN(\bm{\mu}_{\rvz}^{l}(\rve_{\rvz}^{l}, \rvd^{l}), \bm{\sigma}_{\rvz}^{l}(\rve_{\rvz}^{l}, \rvd^{l}))},
\end{equation}
where $\rve_{\rvz}^{l}$ and $\rvd^{l}$ denote the encoding and decoding features at the $l$-th level, respectively. The networks $\bm{\mu}_{\rvz}^{l}$ and $\bm{\sigma}_{\rvz}^{l}$ map these input features to the mean and standard deviation of the Gaussian distribution. The encoding features $\{\rve_{\rvz}^{l}\}_{l=0}^L$ are obtained recursively via
\begin{equation*}
{\textstyle \rve_{\rvz}^{0} = \gE(\rvx), \quad \rve_{\rvz}^{l} = f_{\rvz}^l(\rve_{\rvz}^{l-1}),}
\end{equation*}
where $\gE$ denotes the Input Convolution (InConv) block, and $f_{\rvz}^l$ represents the $l$-th basic encoding block. The decoding features are calculated in reverse order as
\begin{equation*}
{\textstyle \rvd^{L} = \mathbf{0}, \quad \rvd^{l-1} = \varphi^{l}(\rvz^{l} + \rvd^{l}),}
\end{equation*}
where $\varphi^{l}$ is the $l$-th basic decoding block, and $\rvz^l$ is sampled from $\gN(\bm{\mu}_{\rvz}^{l}(\rve_{\rvz}^{l}, \rvd^{l}), \bm{\sigma}_{\rvz}^{l}(\rve_{\rvz}^{l}, \rvd^{l}))$. We design $q(\rvz^l | \rvh_{\rvx},\rvz^{>l})$, $q(\rvz^l | \rvh_{\rvy},\rvz^{>l})$, and $q(\rvz^l | \rvy,\rvz^{>l})$ to share the same structural definition as $q(\rvz^l | \rvx, \rvz^{>l})$. Specifically, $q(\rvz^{l} | \rvh_{\rvx},\rvz^{>l})$ and $q(\rvz^{l} | \rvh_{\rvy},\rvz^{>l})$ share network weights with $q(\rvz^{l} | \rvx,\rvz^{>l})$.

For the $\rvy$-domain branch, we define
\begin{equation}
\label{eq:inference_zn}
{\textstyle q(\rvz_{\rvn}^{>l} | \rvy,\rvz^{\geq l}, \rvz_{\rvn}^{>l}) = \gN({\overline{\bm{\mu}}}_q^l({\rvd}_{\rvn}^{l},\rvz^{l},\rve_{\rvn}^{l}), {\overline{\bm{\sigma}}}_q^l({\rvd}_{\rvn}^{l},\rvz^{l},\rve_{\rvn}^{l}),}
\end{equation}
where $\rve_{\rvn}^{l}$ and $\rvd_{\rvn}^{l}$ denote the respective encoding and decoding features at the $l$-th level. Here, $\rvz^l$ is assumed to be sampled from $\gN( {\bm\mu}_{\rvz}^l(\rve_{\rvn}^{l}, \rvd_{\rvn}^{l}), {\bm\sigma}_{\rvz}^l(\rve_{\rvn}^{l}, \rvd_{\rvn}^{l}) )$, utilizing the parameters ${\bm\mu}_{\rvz}^l$ and ${\bm\sigma}_{\rvz}^l$ defined in~\eqref{eq:inference_z}. The networks ${\overline{\bm{\mu}}}_q^l$ and ${\overline{\bm{\sigma}}}_q^l$ parameterize the mean and variance of the Gaussian distribution. The encoding feature $\{{\rve}_{\rvn}^{l}\}_{l=0}^{L}$ is computed recursively as
\begin{equation*}
{\textstyle \rve_{\rvn}^{0} = \gE(\rvy), \quad \rve_{\rvn}^{l} = f_{\rvn}^{l}(\rve_{\rvn}^{l-1}),}
\end{equation*}
and the decoding feature is obtained via
\begin{equation}
\label{eq:decoding_noisy}
{\textstyle \rvd_{\rvn}^{L} = \mathbf{0}, \quad \rvd_{\rvn}^{l-1} = \varphi^{l}(\rvz_{\rvn}^{l} + \rvz^{l} + \rvd_{\rvn}^{l}),}
\end{equation}
where $\rvz_{\rvn}^{l}$ is sampled from~\eqref{eq:inference_zn}.

For the conditional distribution $p(\rvz_{\rvn}^l | \rvz^{\geq l}, \rvz_{\rvn}^{>l})$, we adopt a structural architecture identical to $q(\rvz_{\rvn}^{l} | \rvy,\rvz^{\geq l}, \rvz_{\rvn}^{>l})$ and define
\begin{equation*}
p(\rvz_{\rvn}^l | \rvz^{\geq l}, \rvz_{\rvn}^{>l}) = \gN( {\overline{\bm\mu}}_p^l(\rvd_{\rvn}^{l},\rvz^{l}), {\overline{\bm\sigma}}_{p}^{l}(\rvd_{\rvn}^{l},\rvz^{l})),
\end{equation*}
where the decoding feature $\rvd_{\rvn}^{l}$ is derived from~\eqref{eq:decoding_noisy}, and the networks $\overline{\bm\mu}_{p}^{l}$ and $\overline{\bm\sigma}_{p}^{l}$ parameterize the generative Gaussian distribution at the $l$-th level.

Finally, for the generative model outputs, we define
\begin{equation*}
{\textstyle p(\rvx | \rvz) = \gN(\gD(\rvd^{0}), \mI), \quad p(\rvy | \rvz, \rvz_{\rvn}) = \gN(\gD(\rvd_{\rvn}^{0}), \mI),}
\end{equation*}
where $\gD$ denotes the Output Convolution (OutConv) block. We adopt the Residual Dense Block (RDB)~\cite{zhang2018residual} as the basic internal unit for $f_{\rvz}^{l}$, $f_{\rvn}^{l}$, and $\varphi^l$.

\noindent\textbf{Controllable degradation generation.} 
When applied to noise modeling, LUD-MSR should be capable of generating images at varying degradation levels from a single clean input. We achieve this by concatenating a target degradation parameter to the inputs of the Gaussian parameter networks $\overline{\bm\mu}_{p}^{l}$, $\overline{\bm\sigma}_p^l$, $\overline{\bm\mu}_q^l$, and $\overline{\bm\sigma}_q^l$, explicitly conditioning the generation process. During training, the standard deviation of the residual $\rvy - \rvh_{\rvy}$ serves as the empirical estimate of the degradation level for $\rvy$. During inference, a user-specified degradation scalar acts as the direct conditional input to $\overline{\bm\mu}_p^l$ and $\overline{\bm\sigma}_p^l$, enabling controllable image synthesis at the desired degradation severity.

\section{Derivation of ELBOs}
\label{appendix:elbo}
\noindent \textbf{Derivation of~\eqref{eq:elbo_x} and~\eqref{eq:elbo_y}.}
According to the proposed probabilistic graphical model,
we have the following decompositions
\begin{equation*}
\begin{aligned}
& p(\rvx | \rvh_{\rvx}) = \E_{p(\rvz | \rvh_{\rvx})}p(\rvx | \rvz), \\
& p(\rvy | \rvh_{\rvy}) = \E_{p(\rvz | \rvh_{\rvy})p(\rvz_{\rvn} | \rvz)}p(\rvy | \rvz, \rvz_{\rvn}).
\end{aligned}
\end{equation*}
Applying the inference models defined in~\eqref{eq:inference_models_unpaired}, we obtain
\begingroup
\begin{equation}
\label{eq:elbo_unpaired_proof_1}
\begin{aligned}
\log p(\rvx | \rvh_{\rvx}) \geq &
\E_{q(\rvz | \rvx, \rvh_{\rvx})}
\log \frac{p(\rvz | \rvh_{\rvx}) p(\rvx | \rvz)}
{q(\rvz | \rvx, \rvh_{\rvx})} \\
= & \E_{q(\rvz | \rvx, \rvh_{\rvx})}\log p(\rvx | \rvz) - \KL(q(\rvz | \rvx, \rvh_{\rvx}) || p(\rvz | \rvh_{\rvx})), \\
\log p(\rvy | \rvh_{\rvy}) \geq &
\E_{q(\rvz | \rvy, \rvh_{\rvy})q(\rvz_{\rvn} | \rvy, \rvz)}
\log \frac{
p(\rvz | \rvh_{\rvy}) p(\rvz_{\rvn} | \rvz) p(\rvy | \rvz,\rvz_{\rvn})
}{
q(\rvz | \rvy, \rvh_{\rvy})q(\rvz_{\rvn} | \rvy, \rvz)
} \\
= & \E_{q(\rvz | \rvy, \rvh_{\rvy})q(\rvz_{\rvn} | \rvy, \rvz)} \log p(\rvy | \rvz, \rvz_{\rvn}) 
- \KL(q(\rvz | \rvy, \rvh_{\rvy}) || p(\rvz | \rvh_{\rvy})) \\
& - \E_{q(\rvz | \rvy, \rvh_{\rvy})}
\KL(q(\rvz_{\rvn} | \rvy, \rvz) || p(\rvz_{\rvn} | \rvz)).
\end{aligned}
\end{equation}
By setting $p(\rvz | \rvh_{\rvx}) = q(\rvz | \rvh_{\rvx})$ and
$p(\rvz | \rvh_{\rvy}) = q(\rvz | \rvh_{\rvy})$ via shared network
parameters,
we establish that
\begin{equation*}
\begin{aligned}
\KL(q(\rvz | \rvx, \rvh_{\rvx}) || p(\rvz | \rvh_{\rvx})) \leq & \alpha \KL(q(\rvz | \rvx) || q(\rvz | \rvh_{\rvx})), \\
\KL(q(\rvz | \rvy, \rvh_{\rvy}) || p(\rvz | \rvh_{\rvy})) \leq & \alpha \KL(q(\rvz | \rvy) || q(\rvz | \rvh_{\rvy})).
\end{aligned}
\end{equation*}
\endgroup
Substituting these inequalities into~\eqref{eq:elbo_unpaired_proof_1} yields~\eqref{eq:elbo_x} and~\eqref{eq:elbo_y}.

\noindent\textbf{Derivation of~\eqref{eq:elbo_paired}.}
Analogous to the derivation of~\eqref{eq:elbo_unpaired_proof_1}, 
\begin{equation*}
\begin{aligned}
\log p(\rvx | \rvy) \geq & \E_{q(\rvz | \rvx, \rvy)}\log p(\rvx | \rvz) - \KL(q(\rvz | \rvx, \rvy) || p(\rvz | \rvy)), \\
\log p(\rvy | \rvx) \geq & \E_{q(\rvz | \rvx, \rvy)q(\rvz_{\rvn} | \rvy, \rvz)} \log p(\rvy | \rvz, \rvz_{\rvn}) 
- \KL(q(\rvz | \rvx, \rvy) || p(\rvz | \rvx)) \\
& - \E_{q(\rvz | \rvx,\rvy)}\KL(q(\rvz_{\rvn} | \rvy, \rvz) || p(\rvz_{\rvn} | \rvz)).
\end{aligned}
\end{equation*}
Similarly, by setting $p(\rvz | \rvx) = q(\rvz | \rvx)$, $p(\rvz | \rvy) = q(\rvz | \rvy)$, and applying the inference model defined in~\eqref{eq:inference_model_paired},
\begingroup
\begin{equation*}
\begin{aligned}
\KL(q(\rvz | \rvx, \rvy) || p(\rvz | \rvy)) \leq & \alpha \KL(q(\rvz | \rvx) || q(\rvz | \rvy)), \\
\KL(q(\rvz | \rvx, \rvy) || p(\rvz | \rvx)) \leq & (1-\alpha) \KL(q(\rvz | \rvy) || q(\rvz | \rvx)).
\end{aligned}
\end{equation*}
\endgroup
Substituting these upper bounds into the preceding inequalities yields~\eqref{eq:elbo_paired}.

\noindent\textbf{Equivalence to ELBO Maximization.} 
We demonstrate that minimizing the unpaired loss $L_{\text{g}}^{\text{(up)}}$ defined in~\eqref{eq:loss_graph_unpaired} is equivalent to maximizing the ELBO of the joint log-likelihood $\log p(\rvx, \rvy)$. This equivalence holds under two assumptions: (i) $\rvh_{\rvx} = \rvh_{\rvy}$ for $(\rvx,\rvy) \sim q_{X,Y}$, and (ii) $p(\rvy | \rvx, \rvh_{\rvx}) = p(\rvy | \rvh_{\rvx})$. These assumptions are direct consequences of the domain consistency~\eqref{eq:msr_domain_consistency} and information preservation~\eqref{eq:msr_information_preservation} conditions, respectively. Denoting the push-forward distributions as $q_{X}^{h} = h_{\#}q_{X}$, $q_{Y}^{h} = h_{\#}q_{Y}$, and $q_{X,Y}^{h} = h_{\#}q_{X,Y}$, we obtain
\begin{equation*}
\begin{aligned}
p(\rvx, \rvy) = & \E_{q_{X,Y}^{h}(\rvh_{\rvx},\rvh_{\rvy})} p(\rvx,\rvy | \rvh_{\rvx}, \rvh_{\rvy}) \\
= & \E_{q_{X,Y}^{h}(\rvh_{\rvx},\rvh_{\rvy})} p(\rvx | \rvh_{\rvx}, \rvh_{\rvy}) p(\rvy | \rvx, \rvh_{\rvx}, \rvh_{\rvy}) \\
= & \E_{q_{X,Y}^{h}(\rvh_{\rvx},\rvh_{\rvy})}p(\rvx | \rvh_{\rvx})p(\rvy | \rvh_{\rvy}).
\end{aligned}
\end{equation*}
Consequently,
\begin{equation*}
\log p(\rvx,\rvy) \geq \E_{q_{X}^{h}(\rvh_{\rvx})}\log p(\rvx | \rvh_{\rvx}) + \E_{q_{Y}^{h}(\rvh_{\rvy})}\log p(\rvy | \rvh_{\rvy}). 
\end{equation*}

\section{Proof of Proposition~\ref{prop:asympotic_consistency}}
\label{sec:app_proof_of_prop_1}

\begin{proof}[Proof of Proposition~\ref{prop:asympotic_consistency}]
For $\rvx\sim q_{X}$, let $p_{\sigma}^{Y}(\cdot | \rvx)$ denote the probability density function of the sum $\rvn + \rvn_{\sigma}$, where $\rvn \sim q_{\rvn | X}(\cdot | \rvx)$ and $\rvn_{\sigma}\sim p_{\sigma}$. By assumption, there exists an $M > 0$ such that $\sup_{\rvu}|h(\rvu)| \leq M$. Consequently, we obtain
\begin{equation*}
\begin{aligned}
& \|\E_{q_{X,Y}(\rvx,\rvy) p_{\sigma}(\rvn_{\rvx}) p_{\sigma}(\rvn_{\rvy})} [h(\rvx + \rvn_{\rvx}) - h(\rvy + \rvn_{\rvy})]\|_{1} \\
\leq & \E_{q_{X}(\rvx)}{\textstyle \int |h(\rvx + \rvn)| \, |p_{\sigma}(\rvn) - p_{\sigma}^{Y}(\rvn | \rvx)|\df\rvn} \\
\leq & {\textstyle M \, \E_{q_{X}(\rvx)}\|p_{\sigma} - p_{\sigma}^{Y}(\cdot | \rvx)\|_{1}} \\
\leq & {\textstyle M \, \E_{q_{X}(\rvx)}\sqrt{2 \, \KL(p_{\sigma} || p_{\sigma}^{Y}(\cdot | \rvx))} }\\
= & {\textstyle \tfrac{M}{\sigma} \,\sqrt{\E_{q_{X,Y}(\rvx,\rvy)}\|\rvy - \rvx\|_2^2}.}
\end{aligned}
\end{equation*}
Clearly, this upper bound approaches $0$ as $\sigma\to\infty$.
\end{proof}

\section{Proof of Proposition~\ref{prop:h_simple_design}}
\label{appendix:proof_h_simple_design}
\begin{proof}[Proof of Proposition~\ref{prop:h_simple_design}]
Given the definition of $\gD$ in~\eqref{eq:h_design} and the specified low-pass filter $\rvk_{\text{low}} = (\frac{1}{4}, \frac{1}{2}, \frac{1}{4})$, $\gD$ operates as a strided band-Toeplitz matrix under Dirichlet boundary conditions, whose non-zero entries are confined to a diagonal band with a two-column shift between adjacent rows. Consequently, $\gD\gD^{\top}$ is a symmetric tridiagonal Toeplitz matrix, with main diagonal entries of $\frac{3}{8}$ and first off-diagonal entries of $\frac{1}{16}$.
By the Gershgorin circle theorem, the eigenvalues of $\gD\gD^{\top}$ lie within $[\frac{1}{4}, \frac{1}{2}]$. This implies that the non-zero singular values of $\gD$ lie within $[\frac{1}{2}, \frac{1}{\sqrt{2}}]$, yielding $2^{-2m}\mI \preceq \gD^m(\gD^{m})^{\top} \preceq 2^{-m}\mI$. Since $\text{rank}(\gD^m) = K/2^{m}$, the matrix $(\gD^{m})^{\top}\gD^{m}$ has $K/2^{m}$ non-zero eigenvalues, each bounded within $[2^{-2m}, 2^{-m}]$. Thus, the matrix $\mI - (\gD^{m})^{\top}\gD^{m}$ has $K(1-2^{-m})$ eigenvalues equal to $1$, while its remaining $K/2^{m}$ eigenvalues lie in $[1-2^{-m}, 1-2^{-2m}]$. Consequently, we obtain the following upper bound on the domain consistency error
\begingroup
\begin{equation*}
\begin{aligned}
\E_{q_{X,Y}(\rvx,\rvy)}\|h(\rvx) - h(\rvy)\|_{2}^2 
= & \Tr\big((\gD^{m})^{\top}\gD^{m}(\gD^{m})^{\top}\gD^{m} \, \E_{q_{\rvn}(\rvn)}[\rvn\rvn^{\top}]\big) \\
\leq & 2^{-m} \, \Tr\big((\gD^{m})^{\top}\gD^{m} \, \E_{q_{\rvn}(\rvn)}[\rvn\rvn^{\top}] \big) \\
\lesssim &  K / 2^{3m}.
\end{aligned}
\end{equation*}
\endgroup
Similarly, we lower bound the information preservation error
\begingroup
\begin{equation*}
\begin{aligned}
\E_{q_{X}(\rvx)}\|h(\rvx) - \rvx\|_2^2 
= & \Tr\big(\big(\mI - (\gD^{m})^{\top}\gD^{m}\big)^2 \, \E_{q_{X}(\rvx)}[\rvx\rvx^{\top}]\big) \\
\gtrsim & K \, (1-2^{-m}) + K\,2^{-m}\left(1-2^{-2m}\right)^2 \\
\asymp &  K \, (1-2^{-3m}),
\end{aligned}
\end{equation*}
\endgroup
where $\asymp$ denotes equivalence up to a constant.
\end{proof}

\section{Architecture of Multi-Scale Image Representation (MSR) Mapping}
\label{sec:archi_msr}
As established in~\eqref{eq:msr_mapping}, our MSR mapping is defined as
\begin{equation*}
h(\rvx) = f^{-1}([f(\rvx)]_{1:K/2^{T}}, \mathbf{0}),
\end{equation*}
where the invertible mapping $f = f_{T} \circ \cdots \circ f_{1}$ is composed of $T=3$ scale levels, with each level defined as $f_{l} = g_{l} \circ \mW$. Here, $\mW = (\mW_{L}^{\top}, \mW_{H}^{\top})^{\top}$ represents a wavelet operator partitioned as
\begin{equation*}
\begin{aligned}
& \mW_{L}: \rvx \mapsto \downarrow_{2} (\rvk_{\text{low}} \ast \rvx ), \\
& \mW_{H}: \rvx \mapsto \big(
\downarrow_{2} (\rvk_{\text{high}}^{(1)} \ast \rvx ); \ldots; 
\downarrow_{2} (\rvk_{\text{high}}^{(r)} \ast \rvx )
\big),
\end{aligned}
\end{equation*}
where $\rvk_{\text{low}}$ denotes a low-pass filter and $\{\rvk_{\text{high}}^{(i)}\}_{i=1}^{r}$ constitutes a set of high-pass filters. In our implementation, we employ the linear B-spline wavelet tight frame~\cite{daubechies2003framelets}, which specifies
\begin{equation*}
\begin{aligned}
& \rvk_{\text{low}} = (\tfrac{1}{4}, \tfrac{1}{2}, \tfrac{1}{4}), \\
& \rvk_{\text{high}}^{(1)} = (-\tfrac{1}{4}, \tfrac{1}{2}, -\tfrac{1}{4}), \, 
\rvk_{\text{high}}^{(2)} = (\tfrac{\sqrt{2}}{4}, 0, -\tfrac{\sqrt{2}}{4}).
\end{aligned}
\end{equation*}
The normalizing flow $g_{l} = g_{l}^{(M)} \circ \cdots \circ g_{l}^{(1)}$ comprises a sequence of $M=8$ flow blocks. Each block $g_{l}^{(i)}$ adopts the architecture of the Glow step proposed in~\cite{kingma2018glow}, consisting of an ActNorm layer, an invertible $1\times 1$ convolution, and an affine coupling layer $\phi_{l}^{(i)}$:
\begin{equation*}
g_{l}^{(i)} = \phi_{l}^{(i)} \circ \text{Inv}_{1\times1} \circ \text{ActNorm}.
\end{equation*}
We refer the reader to~\cite{kingma2018glow} for the implementation details of the ActNorm and $\text{Inv}_{1\times 1}$ layers. Specifically, we formulate the affine coupling layer $\phi_{l}^{(i)}$ as
\begin{equation*}
\phi_{l}^{(i)} \big(\rvc_{L}, \rvc_{H} \big) = 
\begin{cases}
\big(\rvc_{L}, \rvs_{l}^{(i)}(\rvc_{L}) \odot \rvc_{H} + \rvt_{l}^{(i)}(\rvc_{L}) \big) & \text{if $i$ is odd,} \\
\big(\rvc_{L} + \rvt_{l}^{(i)}(\rvc_{H}), \rvc_{H} \big) & \text{if $i$ is even,}
\end{cases}
\end{equation*}
where $\rvc_{L}$ and $\rvc_{H}$ denote the low- and high-frequency components of the input, respectively; $\rvs_{l}^{(i)}$ and $\rvt_{l}^{(i)}$ parameterize the scale and shift coefficients; and $\odot$ denotes the Hadamard product. Because each flow block $g_{l}^{(i)}$ is invertible and the wavelet operator $\mW$ admits a pseudo-inverse $\mW^{\top}$ due to the perfect reconstruction property $\mW^{\top} \mW = \mI$, the overall invertibility of the mapping $f$ is mathematically guaranteed.

\section{Proof of Proposition~\ref{prop:info_trade-off}}
\label{appendix:trade-off_info}


\begin{proof}[Proof of Proposition~\ref{prop:info_trade-off}]
Let $H(X) = -\E_{p(X)} \log p(X)$ denote the Shannon entropy of a random variable $X \sim p$, and let $I(X; Y) = \E_{p(X,Y)} \log \frac{p(X,Y)}{p(X) p(Y)}$ denote the mutual information. The information preservation condition implies
\begin{equation}
\label{eq:information_preservation_info}
I(\rvh_{\rvx}; \rvx) = H(\rvx)
\end{equation}
while the domain consistency condition dictates that $H(\rvh_{\rvx} | \rvy) = H(\rvh_{\rvy} | \rvy) = 0$.
By the data processing inequality,
\begin{equation*}
I(\rvx; \rvy) \geq I(\rvh_{\rvx}; \rvy) = H(\rvh_{\rvx}) - H(\rvh_{\rvx} | \rvy) = H(\rvh_{\rvx}).
\end{equation*}
Since $H(\rvh_{\rvx}) \geq I(\rvh_{\rvx}; \rvx)$, it follows that
\begin{equation}
\label{eq:domain_consistency_info}
I(\rvx; \rvy) \geq I(\rvh_{\rvx}; \rvx).
\end{equation}
Combining~\eqref{eq:information_preservation_info} and~\eqref{eq:domain_consistency_info} yields $H(\rvx | \rvy) 
= H(\rvx) - I(\rvx; \rvy) \leq 0$, which contradicts the assumption that $H(\rvx | \rvy) > 0$.
\end{proof}

\bibliographystyle{plain}
\bibliography{ref}

\end{document}